\documentclass{article}
\usepackage{style/unites}
\usepackage{XCharter}
\usepackage[scaled=1.1]{zlmtt}

\usepackage[utf8]{inputenc}
\usepackage[T1]{fontenc}
\usepackage{microtype}

\usepackage{amsmath}
\usepackage{amssymb}
\usepackage{amsfonts}
\usepackage{amsthm}
\usepackage{mathtools}
\usepackage{mathrsfs}
\usepackage{physics}
\usepackage{braket}
\usepackage{slashed}
\usepackage{nicefrac}
\usepackage{textcomp}
\usepackage{dsfont}
\usepackage{bbm}
\usepackage{bm}

\usepackage{graphicx}
\usepackage{subcaption}
\usepackage[export]{adjustbox}
\usepackage{float}
\usepackage{booktabs}
\usepackage{dcolumn}
\newcolumntype{d}[1]{D{.}{.}{#1}}
\usepackage{bigstrut, tabularx, multirow, makecell, diagbox}
\usepackage{colortbl}
\usepackage{tabularray}
\UseTblrLibrary{booktabs}
\usepackage{threeparttable}
\usepackage{tablefootnote}
\usepackage{fontawesome5}

\usepackage{placeins}
\usepackage{caption}
\usepackage{footnote}
\usepackage{enumitem}
\usepackage{multicol}
\usepackage{xspace}
\usepackage{titletoc}
\usepackage{titlesec}
\usepackage[bottom]{footmisc}
\usepackage{setspace}

\usepackage{wrapfig}
\usepackage{tikz}
\usepackage{quantikz}
\usepackage{dashbox}
\usepackage{mdframed}
\usepackage{marvosym}
\usepackage{pifont}
\usepackage{CJK}
\usepackage{url}

\usepackage[table,x11names]{xcolor}
\usepackage[most]{tcolorbox}
\tcbuselibrary{breakable}
\usetikzlibrary{decorations.pathreplacing, fit}

\definecolor{primaryblue}{HTML}{0066CC}
\definecolor{accentcyan}{HTML}{00D4AA}
\definecolor{warmorange}{HTML}{FF6B35}
\definecolor{deepgray}{HTML}{2C3E50}
\definecolor{lightgray}{HTML}{F8F9FA}
\definecolor{gradientstart}{HTML}{667eea}
\definecolor{gradientend}{HTML}{764ba2}

\definecolor{citecolor}{HTML}{0071bc}
\definecolor{citeblue}{RGB}{0, 113, 188}
\definecolor{linkcolor}{HTML}{9A4D92}
\definecolor{firebrick}{rgb}{0.698,0.133,0.133}

\definecolor{paleviolet}{HTML}{E1EEFC}
\definecolor{CarolinaUltraLight}{HTML}{E7F4FC}
\definecolor{lightgrey}{RGB}{247, 247, 247}
\definecolor{shadecolor}{HTML}{EFEFEF}
\definecolor{lightyellow}{rgb}{1.0, 0.95, 0.7}
\definecolor{lightblue}{rgb}{0.90, 0.95, 1.0}
\definecolor{light-gray}{gray}{0.95}

\definecolor{darkgrey}{rgb}{0.5, 0.5, 0.5}
\definecolor{darkgreen}{rgb}{0, 0.5, 0}
\definecolor{mydarkblue}{rgb}{0,0.08,0.45}
\definecolor{mydarkblue2}{rgb}{0.133, 0.133, 0.698}
\definecolor{echodrk}{HTML}{0099cc}
\definecolor{mymauve}{rgb}{0.58,0,0.82}
\definecolor{midnightblue}{rgb}{0.1,0.1,0.44}
\definecolor{oxfordblue}{rgb}{0.0,0.13,0.28}
\definecolor{prussianblue}{rgb}{0.0,0.19,0.33}
\definecolor{coolteal}{rgb}{0, 0.45, 0.45}
\definecolor{olive}{rgb}{0.1, 0.3, 0}
\definecolor{mypurple}{rgb}{0.5,0,0.5}
\definecolor{almond}{rgb}{0.94, 0.87, 0.8}

\definecolor{blue_ampEncoding}{HTML}{DAE8FC}
\definecolor{green_encoder}{HTML}{D5E8D4}
\definecolor{purple_decoder}{HTML}{E1D5E7}
\definecolor{yellow_measure}{HTML}{FFF2CC}
\definecolor{gray_block}{HTML}{F5F5F5}
\definecolor{pink_dru}{HTML}{FAD9D5}
\definecolor{orange_v}{HTML}{FAD7AC}

\definecolor{colorA}{rgb}{1,0,0}
\definecolor{colorB}{rgb}{0,0.3,1}
\definecolor{colorC}{rgb}{0.9,0.8,0.2}
\definecolor{colorD}{rgb}{0,0.65,0}
\definecolor{lesslightgray}{rgb}{0.5,0.5,0.5}
\definecolor{fundamental}{RGB}{55, 110, 111}
\definecolor{Gred}{RGB}{219, 50, 54}
\definecolor{ToCgreen}{RGB}{0, 128, 0}
\definecolor{Sepia}{RGB}{112, 66, 20}
\definecolor{Dblue}{rgb}{0,0.08,0.45}
\definecolor{Blue}{rgb}{0, 0, 0.8}
\definecolor{blue}{rgb}{0,0,1}
\definecolor{UNCblue!10}{rgb}{0.84,0.91,0.98}
\definecolor{RowAlt}{rgb}{0.98,0.98,0.99}

\definecolor{CarolinaBlue}{HTML}{7BAFD4}        %
\definecolor{CarolinaLightBlue}{HTML}{B3D4E5}   %
\definecolor{CarolinaUltraLight}{HTML}{E8F4F8}  %
\definecolor{CarolinaText}{HTML}{1C2B33}        %

\usepackage[pagebackref=true,breaklinks=true,colorlinks,hyperfootnotes=false]{hyperref}
\hypersetup{
  colorlinks,
  citecolor=citeblue,
  linkcolor=firebrick,
  urlcolor=firebrick
}
\usepackage[nameinlink,capitalize,noabbrev]{cleveref}

\titlespacing\section{0pt}{4pt plus 4pt minus 2pt}{-2pt plus 2pt minus 2pt}
\titlespacing\subsection{0pt}{2pt plus 4pt minus 2pt}{-2pt plus 2pt minus 2pt}
\titlespacing\subsubsection{0pt}{2pt plus 4pt minus 2pt}{-2pt plus 2pt minus 2pt}

\makeatletter
\def\th@remark{%
  \thm@headfont{\bfseries}%
  \normalfont %
  \thm@preskip\topsep \divide\thm@preskip\tw@
  \thm@postskip\thm@preskip
}
\makeatother

\theoremstyle{definition}

\newtheorem{theorem}{Theorem}[section]
\tcolorboxenvironment{theorem}{
  breakable,
  colback=black!10,
  colframe=white,
  width=\linewidth, 
  enlarge left by=0pt,
  enlarge right by=0pt,
  boxsep=5pt,
  boxrule=0pt,
  left=0pt,right=0pt,top=0pt,bottom=0pt,
  arc=8pt,
  before skip=\topsep,
  after skip=\topsep
}

\newtheorem{lemma}{Lemma}[section]
\tcolorboxenvironment{lemma}{
  breakable,
  colback=black!10,
  colframe=white,
  width=\linewidth,
  enlarge left by=0pt,
  enlarge right by=0pt,
  boxsep=5pt,
  boxrule=0pt,
  left=0pt,right=0pt,top=0pt,bottom=0pt,
  arc=8pt,
  before skip=\topsep,
  after skip=\topsep
}

\tcolorboxenvironment{corollary}{
  breakable,
  colback=black!10,
  colframe=white,
  width=\linewidth,
  enlarge left by=0pt,
  enlarge right by=0pt,
  boxsep=5pt,
  boxrule=0pt,
  left=0pt,right=0pt,top=0pt,bottom=0pt,
  arc=8pt,
  before skip=\topsep,
  after skip=\topsep
}

\newtheorem{proposition}{Proposition}[section]
\tcolorboxenvironment{proposition}{
  breakable,
  colback=black!10,
  colframe=white,
  width=\linewidth,
  enlarge left by=0pt,
  enlarge right by=0pt,
  boxsep=5pt,
  boxrule=0pt,
  left=0pt,right=0pt,top=0pt,bottom=0pt,
  arc=8pt,
  before skip=\topsep,
  after skip=\topsep
}

\newtheorem{definition}{Definition}[section]
\tcolorboxenvironment{definition}{
  breakable,
  colback=black!10,
  colframe=white,
  width=\linewidth,
  enlarge left by=0pt,
  enlarge right by=0pt,
  boxsep=5pt,
  boxrule=0pt,
  left=0pt,right=0pt,top=0pt,bottom=0pt,
  arc=8pt,
  before skip=\topsep,
  after skip=\topsep
}

\newtheorem{remark}{Remark}[section]

\tcolorboxenvironment{assumption}{
  breakable,
  colback=black!10,
  colframe=white,
  width=\linewidth,
  enlarge left by=0pt,
  enlarge right by=0pt,
  boxsep=5pt,
  boxrule=0pt,
  left=0pt,right=0pt,top=0pt,bottom=0pt,
  arc=8pt,
  before skip=\topsep,
  after skip=\topsep
}

\tcolorboxenvironment{claim}{
  breakable,
  colback=black!10,
  colframe=white,
  width=\linewidth,
  enlarge left by=0pt,
  enlarge right by=0pt,
  boxsep=5pt,
  boxrule=0pt,
  left=0pt,right=0pt,top=0pt,bottom=0pt,
  arc=8pt,
  before skip=\topsep,
  after skip=\topsep
}

\tcolorboxenvironment{problem}{
  breakable,
  colback=black!10,
  colframe=white,
  width=\linewidth,
  enlarge left by=0pt,
  enlarge right by=0pt,
  boxsep=5pt,
  boxrule=0pt,
  left=0pt,right=0pt,top=0pt,bottom=0pt,
  arc=8pt,
  before skip=\topsep,
  after skip=\topsep
}

\tcolorboxenvironment{question}{
  breakable,
  colback=black!10,
  colframe=white,
  width=\linewidth,
  enlarge left by=0pt,
  enlarge right by=0pt,
  boxsep=5pt,
  boxrule=0pt,
  left=0pt,right=0pt,top=0pt,bottom=0pt,
  arc=8pt,
  before skip=\topsep,
  after skip=\topsep
}

\newtcolorbox{titleblock}{
  enhanced,
  frame hidden,
  colback=CarolinaUltraLight,
  colframe=CarolinaUltraLight,
  boxrule=0pt,
  arc=10pt,
  left=14pt,
  right=14pt,
  top=14pt,
  bottom=14pt,
  width=\linewidth,
  before skip=12pt plus 4pt,
  after skip=12pt plus 4pt,
  grow to left by=1.5pt,
  grow to right by=1.5pt,
  before upper={
    \setlength{\parindent}{0cm}
    \setlength{\parskip}{0.5cm}
  }
}

\crefname{theorem}{Theorem}{Theorems}
\crefname{proposition}{Proposition}{Propositions}
\crefname{lemma}{Lemma}{Lemmas}
\crefname{corollary}{Corollary}{Corollaries}
\crefname{definition}{Definition}{Definitions}
\crefname{assumption}{Assumption}{Assumptions}
\crefname{remark}{Remark}{Remarks}
\crefname{problem}{Problem}{Problems}
\crefname{property}{Property}{property}
\crefname{question}{Question}{Questions}

\numberwithin{equation}{section}
\numberwithin{theorem}{section}
\numberwithin{proposition}{section}
\numberwithin{definition}{section}
\numberwithin{lemma}{section}
\numberwithin{assumption}{section}
\numberwithin{remark}{section}

\def\1{\bm{1}}

\makeatletter
\let\save@mathaccent\mathaccent
\newcommand*\if@single[3]{%
    \setbox0\hbox{${\mathaccent"0362{#1}}^H$}%
    \setbox2\hbox{${\mathaccent"0362{\kern0pt#1}}^H$}%
    \ifdim\ht0=\ht2 #3\else #2\fi
}
\newcommand*\rel@kern[1]{\kern#1\dimexpr\macc@kerna}
\newcommand*\widebar[1]{\@ifnextchar^{{\wide@bar{#1}{0}}}{\wide@bar{#1}{1}}}
\newcommand*\wide@bar[2]{\if@single{#1}{\wide@bar@{#1}{#2}{1}}{\wide@bar@{#1}{#2}{2}}}
\newcommand*\wide@bar@[3]{%
    \begingroup
    \def\mathaccent##1##2{%
        \let\mathaccent\save@mathaccent
        \if#32 \let\macc@nucleus\first@char \fi
        \setbox\z@\hbox{$\macc@style{\macc@nucleus}_{}$}%
        \setbox\tw@\hbox{$\macc@style{\macc@nucleus}{}_{}$}%
        \dimen@\wd\tw@
        \advance\dimen@-\wd\z@
        \divide\dimen@ 3
        \@tempdima\wd\tw@
        \advance\@tempdima-\scriptspace
        \divide\@tempdima 10
        \advance\dimen@-\@tempdima
        \ifdim\dimen@>\z@ \dimen@0pt\fi
        \rel@kern{0.6}\kern-\dimen@
        \if#31
        \overline{\rel@kern{-0.6}\kern\dimen@\macc@nucleus\rel@kern{0.4}\kern\dimen@}%
        \advance\dimen@0.4\dimexpr\macc@kerna
        \let\final@kern#2%
        \ifdim\dimen@<\z@ \let\final@kern1\fi
        \if\final@kern1 \kern-\dimen@\fi
        \else
        \overline{\rel@kern{-0.6}\kern\dimen@#1}%
        \fi
    }%
    \macc@depth\@ne
    \let\math@bgroup\@empty \let\math@egroup\macc@set@skewchar
    \mathsurround\z@ \frozen@everymath{\mathgroup\macc@group\relax}%
    \macc@set@skewchar\relax
    \let\mathaccentV\macc@nested@a
    \if#31
    \macc@nested@a\relax111{#1}%
    \else
    \def\gobble@till@marker##1\endmarker{}%
    \futurelet\first@char\gobble@till@marker#1\endmarker
    \ifcat\noexpand\first@char A\else
    \def\first@char{}%
    \fi
    \macc@nested@a\relax111{\first@char}%
    \fi
    \endgroup
    }
\makeatother

\DeclareMathAlphabet{\mathsfit}{\encodingdefault}{\sfdefault}{m}{sl}
\SetMathAlphabet{\mathsfit}{bold}{\encodingdefault}{\sfdefault}{bx}{n}

\renewcommand{\arraystretch}{1.15}
\newtheorem{ResearchQ}{RQ}
\tcolorboxenvironment{ResearchQ}{
  breakable,
  colback=black!10,
  colframe=white,%
  width=\dimexpr\linewidth\relax,%
  enlarge left by=0pt,%
  enlarge right by=0pt,%
  boxsep=5pt,%
  boxrule=0pt,
  left=0pt,right=0pt,top=0pt,bottom=0pt,
  arc=8pt
  before skip=\topsep,
  after skip=\topsep
}
\crefname{ResearchQ}{ResearchQ}{ResearchQs}

\usepackage{listings}
\lstdefinestyle{plainmarkdown}{
  language={},
  basicstyle=\fontsize{12}{14}\selectfont\ttfamily,
  breaklines=true,
  breakatwhitespace=true,
  showstringspaces=false,
  showspaces=false,
  showtabs=false,
  columns=fullflexible,
  keepspaces=false,
  backgroundcolor=\color{gray!8},
  frame=single,
  rulecolor=\color{gray!40},
  framerule=0.8pt,
  xleftmargin=3pt,
  xrightmargin=3pt,
  framextopmargin=3pt,
  framexbottommargin=3pt,
  breakindent=0pt,
  postbreak=\mbox{\textcolor{red}{$\hookrightarrow$}\space},
}

\lstdefinestyle{markdownhighlight}{
  basicstyle=\ttfamily\footnotesize,
  breaklines=true,
  breakatwhitespace=true,
  showstringspaces=false,
  showspaces=false,
  showtabs=false,
  columns=fullflexible,
  keepspaces=false,
  backgroundcolor=\color{gray!8},
  frame=single,
  rulecolor=\color{gray!40},
  framerule=0.8pt,
  xleftmargin=3pt,
  xrightmargin=3pt,
  framextopmargin=3pt,
  framexbottommargin=3pt,
  breakindent=0pt,
  postbreak=\mbox{\textcolor{red}{$\hookrightarrow$}\space},
  escapeinside={(*@}{@*)},
}

\renewcommand{\arraystretch}{1.15}
\definecolor{UNCblue!10}{rgb}{0.84,0.91,0.98}
\definecolor{RowAlt}{rgb}{0.98,0.98,0.99}

\usepackage{tocloft}

\usepackage{ulem}
\usepackage{listings}
\usepackage{xcolor}

\definecolor{codebg}{HTML}{F6F8FA}
\definecolor{codeframe}{HTML}{D0D7DE}
\definecolor{codekw}{HTML}{0969DA}
\definecolor{codecomment}{HTML}{6E7781}
\definecolor{codestring}{HTML}{0A3069}
\lstdefinelanguage{Solidity}{
  keywords={
    contract,interface,function,returns,external,payable,event,
    address,bool,uint256,bytes32,calldata
  },
  keywordstyle=\color{codekw}\bfseries,
  identifierstyle=\color{black},
  sensitive=true,
  comment=[l]{//},
  morecomment=[s]{/*}{*/},
  commentstyle=\color{codecomment}\itshape,
  stringstyle=\color{codestring},
}
\usepackage{tcolorbox}
\tcbuselibrary{skins, breakable}

\definecolor{codebg}{HTML}{F6F8FA}      %
\definecolor{codeframe}{HTML}{D0D7DE}   %
\definecolor{codekw}{HTML}{0969DA}      %
\definecolor{codecomment}{HTML}{6E7781}
\definecolor{codestring}{HTML}{0A3069}

\tcbset{
  metabox/.style={
    enhanced,
    colback=codebg,
    colframe=codeframe,
    coltitle=black,
    fonttitle=\bfseries\small,
    boxrule=0.5pt,
    arc=3pt,
    left=8pt,
    right=8pt,
    top=6pt,
    bottom=6pt,
    boxsep=0pt,
    breakable,
  }
}

\tcbset{
  metainput/.style={
    metabox,
    colbacktitle=codekw,
    coltitle=white,
    attach boxed title to top left={yshift=-2mm, xshift=4mm},
    boxed title style={
      colback=codekw,
      colframe=codekw,
      arc=2pt,
      boxrule=0pt,
    },
  }
}

\tcbset{
  metatrace/.style={
    metabox,
    title filled,
    colbacktitle=codebg,
    coltitle=black,
    toptitle=4pt,
    bottomtitle=4pt,
    fonttitle=\bfseries\footnotesize,
    separator sign={},
  }
}

\tcbset{
  metaaudit/.style={
    metabox,
    title filled,
    colbacktitle=codebg,
    coltitle=black,
    fonttitle=\bfseries\scriptsize,
    left=4pt,
    right=4pt,
    top=4pt,
    bottom=4pt,
  }
}

\tcbset{
  metapass/.style={
    metaaudit,
    colframe={rgb:green,4;black,1},
    colbacktitle={rgb:green,1;white,9},
  },
  metafail/.style={
    metaaudit,
    colframe={rgb:red,4;black,1},
    colbacktitle={rgb:red,1;white,9},
  }
}

\newcommand{\TRUST}{\textsc{Trust}}

\newcommand{\DAAN}{\textsc{DAAN}}

\begin{document}

\makeatletter
\def\blfootnote{\gdef\@thefnmark{}\@footnotetext}
\makeatother

\makeatletter
\pagestyle{fancy}
\fancyhf{}
\renewcommand{\headrulewidth}{1pt}
\chead{\small\bf \TRUST{}: A Framework for Decentralized AI Service v.0.1

}
\cfoot{\thepage}
\thispagestyle{fancy}
\makeatother

\makeatletter
\def\icmldate#1{\gdef\@icmldate{#1}}
\icmldate{\today}
\makeatother

\makeatletter
\fancypagestyle{fancytitlepage}{
  \fancyhead{}
  \lhead{\includegraphics[height=0.8cm]{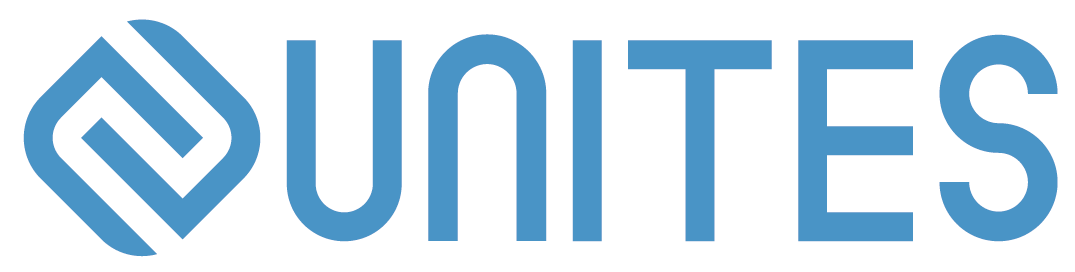}}
  \rhead{\it \@icmldate}
  \cfoot{}
}
\makeatother

\thispagestyle{fancytitlepage}

\vspace*{0.5em}

\noindent
\begin{titleblock}
    {\setlength{\parskip}{0cm}
     \raggedright
     {\setstretch{1.2}
      \LARGE\bfseries
      \TRUST{}: A Framework for Decentralized AI Service v.0.1

      \par}
    }
    \vskip 0.2cm
    
    \begin{icmlauthorlist}
\mbox{Yu-Chao Huang$^{\,1\,\textrm{\Letter}}$},
\mbox{Zhen Tan$^{\,2\,}$}, 
\mbox{Mohan Zhang$^{\,1\,}$}, 
\mbox{Pingzhi Li$^{\,1\,}$}, 
\mbox{Zhuo Zhang$^{\,3\,}$}, 
and \mbox{Tianlong Chen$^{\,1\,\textrm{\Letter}}$}
\end{icmlauthorlist}

$^{1\,}$UNITES Lab, University of North Carolina at Chapel Hill  
\\ 
$^{2\,}$Arizona State University
\\ 
$^{3\,}$Columbia University

$^{\textrm{\Letter}}$ Corresponding Author: \{morris, tianlong\}@cs.unc.edu

    \vskip 0.2cm
    
\noindent
Large Reasoning Models (LRMs) and Multi-Agent Systems (MAS) in high-stakes
domains demand reliable verification, yet centralized approaches suffer four
limitations: (1) Robustness, with single points of failure vulnerable to
attacks and bias; (2) Scalability, as reasoning complexity creates bottlenecks;
(3) Opacity, as hidden auditing erodes trust; and (4) Privacy, as exposed
reasoning traces risk model theft.
We introduce TRUST (Transparent, Robust, and Unified Services for Trustworthy
AI), a decentralized framework with three innovations: (i) Hierarchical
Directed Acyclic Graphs (HDAGs) that decompose Chain-of-Thought reasoning into
five abstraction levels for parallel distributed auditing; (ii) the DAAN
protocol, which projects multi-agent interactions into Causal Interaction
Graphs (CIGs) for deterministic root-cause attribution; and (iii) a multi-tier
consensus mechanism among computational checkers, LLM evaluators, and human
experts with stake-weighted voting that guarantees correctness under 30\%
adversarial participation.
We prove a Safety-Profitability Theorem ensuring honest auditors profit while
malicious actors incur losses. All decisions are recorded on-chain, while
privacy-by-design segmentation prevents reconstruction of proprietary logic.
Across multiple LLMs and benchmarks, TRUST attains 72.4\% accuracy (4-18\% above
baselines) and remains resilient against 20\% corruption. DAAN reaches 70\%
root-cause attribution (vs. 54-63\% for standard methods) with 60\% token
savings. Human studies validate the design (F1 = 0.89, Brier = 0.074).
The framework supports (A1) decentralized auditing, (A2) tamper-proof
leaderboards, (A3) trustless data annotation, and (A4) governed autonomous
agents, pioneering decentralized AI auditing for safe, accountable deployment
of reasoning-capable systems.

\end{titleblock}

\clearpage
\vspace{1em}
{\LARGE \bfseries Table of Contents}

{
\setlength{\parskip}{-0em}
\startcontents[sections]
\printcontents[sections]{}{1}{}
}

\setlength{\parskip}{.5em}
\clearpage

\section{Introduction}\label{sec:intro}

\subsection{The Centralization Crisis in AI Verification}

The capabilities of large language models have expanded from text generation to complex, multi-step reasoning, leading to the emergence of Large Reasoning Models (LRMs) that produce explicit reasoning traces~\citep{wei2022chain,jaech2024openai,guo2025deepseek}. Simultaneously, the shift from single-model prompting to \emph{agent orchestration} has rapidly expanded what AI systems can accomplish, from tool-augmented reasoning~\citep{yao2023react,schick2023toolformer} to end-to-end software engineering~\citep{jimenez2024swebench,hong2023metagpt}, autonomous scientific discovery~\citep{boiko2023autonomous}, and complex multi-step workflows~\citep{openai2023gpt4}.

While explicit reasoning traces offer unprecedented visibility into AI decision-making processes, they also expose potential flaws: logical errors, lack of faithfulness to the model's true internal state~\citep{turpin2023language}, and safety vulnerabilities. The verification of these intermediate steps is a critical prerequisite for safe deployment in high-stakes domains such as medicine~\citep{singhal2023large}, law~\citep{chalkidis2021lexglue}, and finance~\citep{wang2023fingpt}. The urgency is underscored by emerging regulatory frameworks such as the EU AI Act~\citep{com2021laying} and the NIST AI RMF~\citep{ai2023artificial}, which mandate rigorous documentation and monitoring.

Yet the AI ecosystem remains dominated by centralized providers who control model access, define safety standards, and adjudicate disputes, all without transparency or accountability. This centralization creates systemic risks:

Centralized AI auditing faces five fundamental limitations that motivate our decentralized approach.

\textbf{Single points of failure.} When a major AI provider experiences an outage, thousands of downstream applications fail simultaneously. Centralized auditors constitute ``single points of failure'' that are vulnerable to targeted attacks such as prompt injection~\citep{zou2023universal,perez2022ignore} and susceptible to systemic biases~\citep{bender2021dangers,liang2022holistic}.

\textbf{Opacity and trust erosion.} Users cannot verify claims about model capabilities, safety measures, or data handling. Internal auditing processes at proprietary model providers erode public trust and prevent independent verification of safety claims, conflicting with established principles of transparent reporting~\citep{bommasani2023foundation,mitchell2019model}.

\textbf{Scalability bottlenecks.} The volume and combinatorial complexity of reasoning traces from modern LRMs, especially those employing branching search~\citep{lightman2023let,yao2023tree}, make comprehensive manual verification practically and economically infeasible, as evidenced by the massive human effort required for existing process supervision datasets~\citep{bai2022training}.

\textbf{Privacy-transparency tension.} Exposing complete reasoning traces for public audit risks intellectual property theft through model distillation~\citep{carlini2021extracting} and increases the surface area for extracting sensitive training data~\citep{nasr2023scalable}. Current approaches force a difficult trade-off between transparency and proprietary protection.

\textbf{The ``Black Box of Black Boxes'' problem.} For Multi-Agent Systems, when a swarm fails, the observed error is often far downstream of its cause, buried under layers of interaction and message passing. Recent empirical analyses reveal failure rates of 41\% to 87\% on standard benchmarks~\citep{cemri2025multifail}, with failures clustering into specification issues, inter-agent misalignment, and inadequate task verification.

\subsection{The Inadequacy of Current Evaluation Paradigms}

Current evaluation and alignment practice largely treats AI systems as \emph{linear transcripts}, scoring final outputs via ``LLM-as-a-Judge'' or single automated auditors~\citep{zheng2023judging,gu2024survey}. While convenient and scalable, this paradigm introduces critical vulnerabilities.

\textbf{Centralized Judge Biases.} LLM-as-a-Judge exhibits well-documented biases: position bias favoring certain response orderings~\citep{wang2024positionbias}, verbosity bias toward longer outputs~\citep{ye2024calm}, and self-enhancement bias when evaluating outputs from the same model family~\citep{chen2024humansorllms,panickssery2024selfeval}.

\textbf{The Attribution Gap.} In collaborative swarms, downstream agents frequently inherit upstream faults. Linear audits exhibit \emph{recency bias}, blaming the last agent that touched the answer rather than the true root cause. This mirrors the classic ``fault propagation'' problem in distributed systems~\citep{jha2024causalfault}, where errors at one node cascade through dependencies, yet observability tools only surface symptoms at the terminal node.

\textbf{Protocol Opacity.} Many failures are not ``wrong reasoning'' but broken interfaces, format mismatches, ignored constraints, or semantic drift during transmission~\citep{cemri2025multifail}. These \emph{edge-level} violations between agents are invisible to output-only evaluators, which cannot distinguish whether an agent produced incorrect reasoning or simply received corrupted input.

Multi-agent debate and ensemble voting methods~\citep{chan2023chateval,du2024multiagentdebate,harrasse2024d3} partially address single-judge brittleness but do not resolve the provenance problem, they still exhibit recency bias, blaming the last agent rather than the root cause.

\subsection{Motivating Example: Why Semantic Audit is Necessary}\label{sec:example}

To illustrate concretely why output-only auditing is insufficient, we present a clinical scenario where two reasoning models produce \textit{identical correct outputs}, yet one uses fundamentally flawed reasoning. This demonstrates that output-only auditing cannot distinguish sound evidence-based reasoning from error-prone reasoning that happens to reach the correct answer by coincidence.

\begin{tcolorbox}[colback=gray!10!white,colframe=black,title=\textbf{Clinical Input},boxsep=2pt]
\small
\textbf{Patient Note:} 58-year-old male admitted with atrial fibrillation. Weight: 85 kg; serum creatinine: 1.4 mg/dL; no active bleeding; no history of stroke.

\textbf{Clinical Note:} Patient has hypertension (on BP medications) and type 2 diabetes (on metformin). No heart failure, no prior stroke/TIA, no vascular disease.

\textbf{Task:} Calculate CHA$_2$DS$_2$-VASc score to determine anticoagulation need using the following scoring rules:
\begin{itemize}
    \item Congestive heart failure: +1 point if present
    \item Hypertension: +1 point if present
    \item Age $\geq$75 years: +2 points if applicable
    \item Age 65-74 years: +1 point if applicable
    \item Diabetes mellitus: +1 point if present
    \item Prior stroke/TIA: +2 points if present
    \item Vascular disease: +1 point if present
    \item Female sex: +1 point if female
\end{itemize}

\textbf{Correct Answer:} CHA$_2$DS$_2$-VASc score = 2 (Hypertension +1, Diabetes +1) $\rightarrow$ Anticoagulation recommended
\end{tcolorbox}

Figure~\ref{fig:reasoning-comparison} presents two complete reasoning traces that both arrive at the correct score of 2, demonstrating that \textbf{correct outputs do not guarantee sound reasoning}. The sound reasoning trace systematically evaluates all eight clinical rules, extracting evidence from the appropriate sources and applying each rule independently. In contrast, the flawed reasoning trace contains four critical errors: (1) confused age with weight and creatinine values, arriving at the correct 0 points through wrong logic, (2) completely skipped the Age 65-74 rule, (3) inferred hypertension from elevated creatinine instead of reading the explicit clinical note, and (4) combined vascular disease and sex rules instead of evaluating them independently.

\textbf{Both reasoning traces produce identical, clinically correct outputs} (score = 2, recommend anticoagulation), yet they follow fundamentally different reasoning processes. Output-only auditing, which evaluates only the final score and recommendation, passes both traces as correct. This creates a critical safety gap: the flawed reasoning model would be approved for clinical deployment despite containing four systematic errors. The flawed model arrives at the correct answer only through \textbf{multiple coincidences} that happen to align for this specific patient profile.

\begin{figure*}[t]
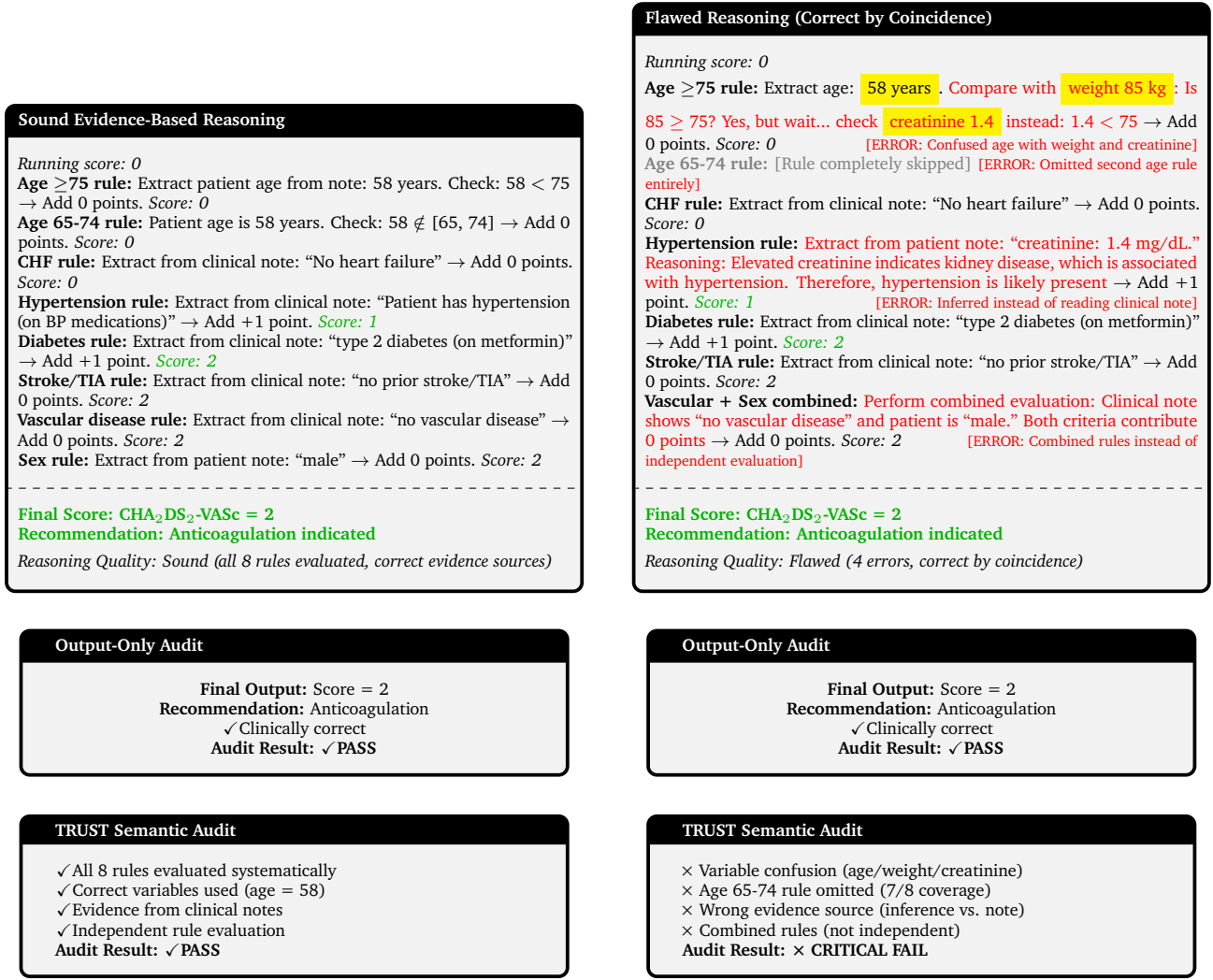

\centering
\scriptsize
\begin{minipage}[t]{0.48\textwidth}
\begin{tcolorbox}[colback=gray!10!white,colframe=black,title=\textbf{Sound Evidence-Based Reasoning},boxsep=2pt,left=2pt,right=2pt]

\textit{Running score: 0}

\textbf{Age $\geq$75 rule:} Extract patient age from note: 58 years. Check: 58 $<$ 75 $\rightarrow$ Add 0 points. 
\textit{Score: 0}

\textbf{Age 65-74 rule:} Patient age is 58 years. Check: 58 $\notin$ [65, 74] $\rightarrow$ Add 0 points.
\textit{Score: 0}

\textbf{CHF rule:} Extract from clinical note: ``No heart failure'' $\rightarrow$ Add 0 points.
\textit{Score: 0}

\textbf{Hypertension rule:} Extract from clinical note: ``Patient has hypertension (on BP medications)'' $\rightarrow$ Add +1 point.
\textcolor{green!70!black}{\textit{Score: 1}}

\textbf{Diabetes rule:} Extract from clinical note: ``type 2 diabetes (on metformin)'' $\rightarrow$ Add +1 point.
\textcolor{green!70!black}{\textit{Score: 2}}

\textbf{Stroke/TIA rule:} Extract from clinical note: ``no prior stroke/TIA'' $\rightarrow$ Add 0 points.
\textit{Score: 2}

\textbf{Vascular disease rule:} Extract from clinical note: ``no vascular disease'' $\rightarrow$ Add 0 points.
\textit{Score: 2}

\textbf{Sex rule:} Extract from patient note: ``male'' $\rightarrow$ Add 0 points.
\textit{Score: 2}

\tcblower

\textcolor{green!70!black}{\textbf{Final Score: CHA$_2$DS$_2$-VASc = 2}}

\textcolor{green!70!black}{\textbf{Recommendation: Anticoagulation indicated}}

\vspace{0.1cm}
\textit{Reasoning Quality: Sound (all 8 rules evaluated, correct evidence sources)}

\end{tcolorbox}
\end{minipage}
\hfill
\begin{minipage}[t]{0.48\textwidth}
\begin{tcolorbox}[colback=gray!10!white,colframe=black,title=\textbf{Flawed Reasoning (Correct by Coincidence)},boxsep=2pt,left=2pt,right=2pt]

\textit{Running score: 0}

\textbf{Age $\geq$75 rule:} Extract age: \colorbox{yellow}{58 years}. \textcolor{red}{Compare with \colorbox{yellow}{weight 85 kg}: Is 85 $\geq$ 75? Yes, but wait... check \colorbox{yellow}{creatinine 1.4} instead: 1.4 $<$ 75} $\rightarrow$ Add 0 points. 
\textit{Score: 0}
\textcolor{red}{\tiny \hfill [ERROR: Confused age with weight and creatinine]}

\textcolor{gray}{\textbf{Age 65-74 rule:} [Rule completely skipped]}
\textcolor{red}{\tiny \hfill [ERROR: Omitted second age rule entirely]}

\textbf{CHF rule:} Extract from clinical note: ``No heart failure'' $\rightarrow$ Add 0 points.
\textit{Score: 0}

\textbf{Hypertension rule:} \textcolor{red}{Extract from patient note: ``creatinine: 1.4 mg/dL.'' Reasoning: Elevated creatinine indicates kidney disease, which is associated with hypertension. Therefore, hypertension is likely present} $\rightarrow$ Add +1 point.
\textcolor{green!70!black}{\textit{Score: 1}}
\textcolor{red}{\tiny \hfill [ERROR: Inferred instead of reading clinical note]}

\textbf{Diabetes rule:} Extract from clinical note: ``type 2 diabetes (on metformin)'' $\rightarrow$ Add +1 point.
\textcolor{green!70!black}{\textit{Score: 2}}

\textbf{Stroke/TIA rule:} Extract from clinical note: ``no prior stroke/TIA'' $\rightarrow$ Add 0 points.
\textit{Score: 2}

\textbf{Vascular + Sex combined:} \textcolor{red}{Perform combined evaluation: Clinical note shows ``no vascular disease'' and patient is ``male.'' Both criteria contribute 0 points} $\rightarrow$ Add 0 points.
\textit{Score: 2}
\textcolor{red}{\tiny \hfill [ERROR: Combined rules instead of independent evaluation]}

\tcblower

\textcolor{green!70!black}{\textbf{Final Score: CHA$_2$DS$_2$-VASc = 2}}

\textcolor{green!70!black}{\textbf{Recommendation: Anticoagulation indicated}}

\vspace{0.1cm}
\textit{Reasoning Quality: Flawed (4 errors, correct by coincidence)}

\end{tcolorbox}
\end{minipage}

\vspace{0.3cm}

\begin{center}
\begin{minipage}{0.48\textwidth}
\centering
\begin{tcolorbox}[colback=gray!10!white,colframe=black,title=\textbf{Output-Only Audit},width=0.95\textwidth,boxsep=2pt]
\centering
\scriptsize
\textbf{Final Output:} Score = 2 \\
\textbf{Recommendation:} Anticoagulation \\
\checkmark Clinically correct \\
\textbf{Audit Result: \checkmark PASS}
\end{tcolorbox}
\end{minipage}
\hfill
\begin{minipage}{0.48\textwidth}
\centering
\begin{tcolorbox}[colback=gray!10!white,colframe=black,title=\textbf{Output-Only Audit},width=0.95\textwidth,boxsep=2pt]
\centering
\scriptsize
\textbf{Final Output:} Score = 2 \\
\textbf{Recommendation:} Anticoagulation \\
\checkmark Clinically correct \\
\textbf{Audit Result: \checkmark PASS}
\end{tcolorbox}
\end{minipage}
\end{center}

\vspace{0.15cm}

\begin{center}
\begin{minipage}{0.48\textwidth}
\centering
\begin{tcolorbox}[colback=gray!10!white,colframe=black,title=\textbf{TRUST Semantic Audit},width=0.95\textwidth,boxsep=2pt]
\scriptsize
\checkmark All 8 rules evaluated systematically \\
\checkmark Correct variables used (age = 58) \\
\checkmark Evidence from clinical notes \\
\checkmark Independent rule evaluation \\
\textbf{Audit Result: \checkmark PASS}
\end{tcolorbox}
\end{minipage}
\hfill
\begin{minipage}{0.48\textwidth}
\centering
\begin{tcolorbox}[colback=gray!10!white,colframe=black,title=\textbf{TRUST Semantic Audit},width=0.95\textwidth,boxsep=2pt]
\scriptsize
\texttimes\ Variable confusion (age/weight/creatinine) \\
\texttimes\ Age 65-74 rule omitted (7/8 coverage) \\
\texttimes\ Wrong evidence source (inference vs. note) \\
\texttimes\ Combined rules (not independent) \\
\textbf{Audit Result: \texttimes\ CRITICAL FAIL}
\end{tcolorbox}
\end{minipage}
\end{center}

\caption{
Comparison of sound clinical reasoning versus flawed reasoning that produces the correct CHA$_2$DS$_2$-VASc score (2) by coincidence.
Both traces generate identical outputs and pass output-only auditing. 
However, TRUST semantic auditing detects four critical errors in the flawed trace: 
(1) confused age with weight and creatinine values, 
(2) completely skipped the Age 65-74 rule, 
(3) inferred hypertension from elevated creatinine instead of reading the explicit clinical note, and 
(4) combined vascular disease and sex rules instead of evaluating independently. 
These errors remain hidden under output-only auditing but cause catastrophic failures under distribution shift.
}
\label{fig:reasoning-comparison}
\end{figure*}

\paragraph{Why Flawed Reasoning is Dangerous.}
Despite producing the correct score, the flawed reasoning contains four critical errors that happen to cancel out only under specific input conditions. These coincidences break catastrophically under a distribution shift. For instance, if the patient's age were 75 rather than 58, the variable confusion would yield incorrect results. If creatinine were normal (0.9 mg/dL), the inference-based hypertension detection would fail. The flawed model would be deployed with apparent 100\% accuracy on the validation set, only to cause patient harm when encountering slightly different clinical profiles.

\paragraph{Token Billing Without Value.}
The token usage patterns reveal an additional concern beyond correctness. The flawed reasoning trace consumes more tokens than the sound reasoning trace despite following an inferior reasoning process. This excess token usage includes wasted computation on incorrect variable comparisons, unnecessary inferential reasoning from creatinine levels rather than direct evidence extraction, and incomplete combined rule evaluation. Consequently, the hospital pays more for reasoning that arrives at the correct answer only through fortunate coincidences rather than sound clinical logic.

\paragraph{Implications for Semantic Auditing.}
This example demonstrates three critical gaps in output-only auditing that \TRUST{} semantic auditing addresses. First, output-only auditing \textbf{cannot detect correct-by-coincidence reasoning}. When both models produce a CHA$_2$DS$_2$-VASc score of 2, output-only auditing passes both as correct without examining their underlying reasoning. In contrast, \TRUST{} semantic auditing detects the variable confusion, skipped rules, and wrong evidence sources in the flawed model, appropriately failing it before deployment.

Second, output-only auditing \textbf{cannot verify billing integrity}. The hospital pays more in reasoning tokens for the flawed model without any means to verify whether the reasoning is sound or merely fortunate. \TRUST{} semantic auditing enables verification of token usage quality, allowing healthcare organizations to ensure they are paying for legitimate reasoning rather than systematic errors that happen to produce correct outputs.

Third, output-only auditing \textbf{cannot ensure regulatory compliance}. The FDA requires ``explainable clinical decision support systems'' that provide transparent reasoning for medical recommendations. Output-only auditing can confirm that a score is 2, but cannot explain how or why that score was calculated. In contrast, \TRUST{} semantic auditing provides a complete audit trail showing the evaluation of each clinical rule, the evidence sources used, and the logic applied. This capability is essential not only for regulatory approval but also for post-incident analysis when adverse outcomes occur.

\subsection{Research Questions and Vision}

Addressing these simultaneous challenges of robustness, scalability, opacity, and privacy demands a fundamentally new approach to the auditing paradigm. Our work is guided by the following research questions:

\begin{ResearchQ}
\textit{How can we design an auditing system that is robust to malicious participants and systemic bias without relying on a central trusted authority?}
\end{ResearchQ}

\begin{ResearchQ}
\textit{How can this system scale to audit complex reasoning traces while preserving the intellectual property of the model provider and ensuring public transparency?}
\end{ResearchQ}

\begin{ResearchQ}
\textit{How can we achieve deterministic root-cause attribution in multi-agent systems where errors propagate through complex interaction patterns?}
\end{ResearchQ}

We envision a future where AI verification operates like other internet protocols: open, permissionless, and resilient. In this future, anyone can provide auditing services by staking capital and demonstrating capability, and anyone can consume audited AI outputs with cryptographic guarantees of quality. Safety and alignment are enforced by economic incentives rather than trust, ensuring that no single entity can censor, manipulate, or shut down the verification network. Root causes of failures are deterministically attributed through causal analysis, closing the loop between detection and accountability.

\TRUST{} is the infrastructure layer that makes this vision possible.

\subsection{The \TRUST{} Framework}

We introduce \TRUST{} (\textbf{T}ransparent, \textbf{R}obust, and \textbf{U}nified \textbf{S}ervices for \textbf{T}rustworthy AI), a decentralized framework for auditing LLM reasoning and multi-agent systems.

To achieve \textbf{robustness}, \TRUST{} establishes a consensus mechanism among a diverse, multi-tier set of auditors, drawing on principles from Byzantine Fault Tolerant systems~\citep{castro1999practical,lamport2019byzantine} to provably guarantee audit correctness even with a significant fraction of malicious participants.

For \textbf{scalability}, the framework introduces a novel decomposition method that transforms reasoning traces into \textit{Hierarchical Directed Acyclic Graphs (HDAGs)}, a structured representation with five abstraction levels, Goal, Strategy, Tactic, Step, and Operation, that permits parallel verification of atomic reasoning steps by a distributed network. This representation is problem-agnostic (mathematics, science, programming, general reasoning) and enables different reasoning components to be audited at appropriate granularity by matched auditor types.

To address the \textbf{multi-agent attribution problem}, we integrate the \DAAN{} protocol (Decentralized Audit and Active Refinement), which projects interaction logs into \textit{Causal Interaction Graphs (CIGs)}. Unlike standard call graphs, CIGs explicitly map \textit{intent} against \textit{outcome}, enabling a dual-layer audit that separately verifies \emph{node validity} (reasoning correctness) and \emph{edge consistency} (protocol adherence and transmission fidelity). This separation enables deterministic fault localization, distinguishing root causes from cascade errors.

To jointly address \textbf{opacity} and \textbf{privacy}, all verification decisions are recorded on a transparent blockchain ledger for public auditability, while the protocol preserves confidentiality through privacy-by-design segmentation. Each auditor receives only atomic segments of the trace without access to complete context or final conclusions, preventing intellectual property theft or model distillation while enabling comprehensive verification through distributed consensus.

Finally, the causal structure enables \textbf{active refinement}: upon identifying root causes, only the faulty subgraph is pruned and regenerated through an Audit-Prune-Regenerate loop, avoiding expensive global retries while preserving valid intermediate work, yielding up to 99\% cost savings for leaf-level errors.

\subsection{\TRUST{} Applications}

\TRUST{} is initially a decentralized auditing framework for large reasoning models. We extend it to a broader framework for decentralized AI services on blockchain, providing solutions for many existing issues in private or centralized AI services:

\begin{enumerate}[leftmargin=2.3em]
\setlength\itemsep{0.3em}
    \item[(A1)] \textbf{Decentralized Auditing.} A decentralized framework for auditing trustworthiness and safety of AI reasoning, increasing reliability of open-source models from 45\% to 72.4\% by filtering ``correct answer, wrong reason'' hallucinations. (See Section~\ref{sec:app_audit})
    
    \item[(A2)] \textbf{Decentralized Model Leaderboard.} Creating tamper-proof model leaderboards resistant to selective release scandals and vote manipulation, with scores minted on blockchain that cannot be altered or cherry-picked. (See Section~\ref{sec:app_model})
    
    \item[(A3)] \textbf{Decentralized Data Annotation.} A trustless collaborative data annotation platform for crowd-sourcing high-quality RLHF training data, creating ``Proof-of-Quality'' dataset markets. (See Section~\ref{sec:app_annot})
    
    \item[(A4)] \textbf{Decentralized Agent Governance.} A decentralized agent pipeline collaboration and deployment platform with runtime guardrails, before an agent executes sensitive tools, actions must be approved by the auditing network, and self-healing swarms that detect loops and self-correct. (See Section~\ref{sec:app_agent})
\end{enumerate}

\subsection{Contributions Summary}

Our main contributions are:

\begin{itemize}[leftmargin=*]
    \item \textbf{First Decentralized Auditing System.} We introduce \TRUST{}, the first decentralized auditing system for reasoning traces that achieves privacy-preserving verification without exposing proprietary models, with provable guarantees under Byzantine fault conditions.
    
    \item \textbf{Novel Graph Decomposition Methods.} We develop systematic approaches to decompose Chain-of-Thought reasoning into Hierarchical Directed Acyclic Graphs (HDAGs) and multi-agent interactions into Causal Interaction Graphs (CIGs), enabling modular verification with deterministic fault localization.
    
    \item \textbf{Multi-Tier Verification Architecture.} We design a three-tier auditor system (Computational/LLM/Human) with stake-weighted consensus and cryptographic commit-reveal voting that routes verification tasks to appropriate auditor types based on complexity.
    
    \item \textbf{Theoretical Foundations.} We prove the Safety-Profitability Theorem, providing rigorous guarantees that honest auditors profit while malicious actors incur losses, with failure probabilities bounded by Hoeffding and Chernoff inequalities.
    
    \item \textbf{Active Refinement.} We enable surgical repair through Prune-Freeze-Repair cycles that regenerate only faulty subgraphs, reducing token costs by 60\% compared to global retry while preserving solution integrity.
    
    \item \textbf{Comprehensive Empirical Validation.} We conduct experiments on diverse datasets (MMLU-Pro, GSM8K, MATH, HumanEval, WritingPrompts) and models (GPT-OSS, DeepSeek-R1, Qwen), including human-in-the-loop studies with 30 participants, demonstrating effectiveness and robustness against centralized baselines.
\end{itemize}

\subsection{Paper Organization}

The remainder of this whitepaper is organized as follows.  Section~\ref{sec:audit} presents the decentralized auditing framework, including HDAG decomposition for single-model reasoning and CIG reconstruction for multi-agent systems. Section~\ref{sec:network} details the network design, including the three-tier auditor architecture and consensus mechanisms. Section~\ref{sec:theory} provides theoretical guarantees for statistical safety and economic sustainability. 
Section~\ref{sec:applications} describes the four primary applications. Section~\ref{sec:discussion} discusses implementation considerations, limitations, and future directions. Section~\ref{sec:conclusion} concludes the whitepaper.

\subsection{Elementary Components}\label{sec:components}

\paragraph{Decentralized Auditing Network (DAN).}
TRUST establishes a heterogeneous auditor network comprising computational checkers, LLM-based evaluators, and human experts. This multi-tier architecture routes verification tasks based on segment complexity, implementing stake-weighted Proof-of-Stake consensus adapted for AI auditing. Auditors operate as anonymous seats that receive only partial trace segments, preventing the reconstruction of proprietary logic while enabling distributed verification. Performance-based rewards and cryptographic slashing mechanisms ensure honest participation remains profitable while making coordinated attacks economically infeasible.

\paragraph{Novel Graph-Based Decomposition.}
TRUST introduces Hierarchical Directed Acyclic Graphs (HDAGs) to transform linear Chain-of-Thought traces into structured, independently auditable components across five abstraction levels: \textit{Goal}, \textit{Strategy}, \textit{Tactic}, \textit{Step}, and \textit{Operation}. Each HDAG node carries metadata specifying content, complexity score, and assigned auditor type, while edges encode logical relationships (\textit{decomposes\_to}, \textit{depends\_on}, \textit{enables}, \textit{validates}, \textit{contradicts}). This representation enables parallel verification of independent reasoning branches while preserving dependency constraints, generalizing across mathematical proofs, scientific reasoning, and open-domain problem-solving.

\paragraph{Multi-Tier Verification Architecture.}
Verification proceeds through three hierarchical layers of consensus. At the \textbf{seat layer}, individual auditors cast binary votes with type-specific error rates $\epsilon_t$ and adversarial probabilities $\rho_t$. At the \textbf{segment layer}, votes aggregate through weighted quorum thresholds $q_t = \lceil \tau k_t \rceil$. At the \textbf{trace layer}, segment outcomes combine through weighted voting $W = \sum_{s=1}^S w_{t(s)} B_s$, with failure probability bounded by Hoeffding and Chernoff inequalities. This architecture achieves graceful degradation under adversarial conditions, maintaining partial verification capability even when minority auditor coalitions coordinate attacks.

\paragraph{Blockchain Consensus Mechanism.}
TRUST employs a permissioned blockchain with Proof-of-Stake consensus to maintain tamper-resistant audit trails while preserving privacy. Smart contracts orchestrate session initialization, stake-weighted auditor assignment, cryptographic commit-reveal voting, and reward distribution. The blockchain records only metadata (session identifiers, segment hashes, vote commitments) while storing full reasoning traces off-chain in IPFS as encrypted, content-addressed objects. The commit-reveal protocol prevents vote manipulation through temporal separation, eliminating passive free-riding and reducing herding behavior.

\subsection{Protocol Overview}\label{sec:protocol_overview}

The TRUST verification workflow executes through six sequential phases: \textbf{(1) Trace Submission and Batching}, providers submit reasoning traces that are batched to anonymize source identity; \textbf{(2) HDAG Decomposition}, traces undergo automated decomposition into graph structures with complexity-based auditor assignment; \textbf{(3) Auditor Assignment and Segment Distribution}, smart contracts select auditor sets using stake-weighted random sampling and distribute encrypted segments via IPFS; \textbf{(4) Commit-Reveal Voting}, auditors submit cryptographic vote commitments followed by verified reveals; \textbf{(5) Consensus Aggregation}, votes aggregate through three-tier hierarchy to produce final verdicts recorded on-chain; \textbf{(6) Reward Distribution}, accurate auditors receive rewards while malicious participants face stake slashing based on performance metrics and honeypot validation.

This architecture addresses the tension between transparency and privacy in proprietary AI auditing. By recording verification decisions publicly while distributing only disconnected trace fragments to individual auditors, TRUST enables verifiable oversight of commercial reasoning systems without compromising intellectual property protections, critical for real-world deployment where model internals constitute trade secrets.

\clearpage
\section{Decentralized Auditing Framework}\label{sec:audit}

Auditing is the fundamental primitive of the \TRUST{} network. Current evaluation paradigms treat AI systems as linear transcripts, scoring final outputs via ``LLM-as-a-Judge'' or single automated auditors~\citep{zheng2023judging,gu2024survey}. While convenient and scalable, this approach introduces critical vulnerabilities: centralized judges exhibit well-documented biases, including position bias~\citep{wang2024positionbias}, verbosity bias~\citep{ye2024calm}, and self-enhancement bias~\citep{chen2024humansorllms}. More fundamentally, output-only evaluation suffers from an \textit{attribution gap}---when downstream agents inherit upstream faults, linear audits exhibit recency bias, blaming the last agent rather than the true root cause~\citep{cemri2025multifail}.

To achieve decentralized verification that is both scalable and privacy-preserving, \TRUST{} moves from the ``linear audit'' to the ``structural decomposition'' paradigm. This section presents two complementary decomposition methods: \textbf{Hierarchical Directed Acyclic Graphs (HDAGs)} for single-model reasoning traces, and \textbf{Causal Interaction Graphs (CIGs)} for multi-agent systems.

\begin{figure}[h]
    \centering
    \includegraphics[width=1.0\textwidth]{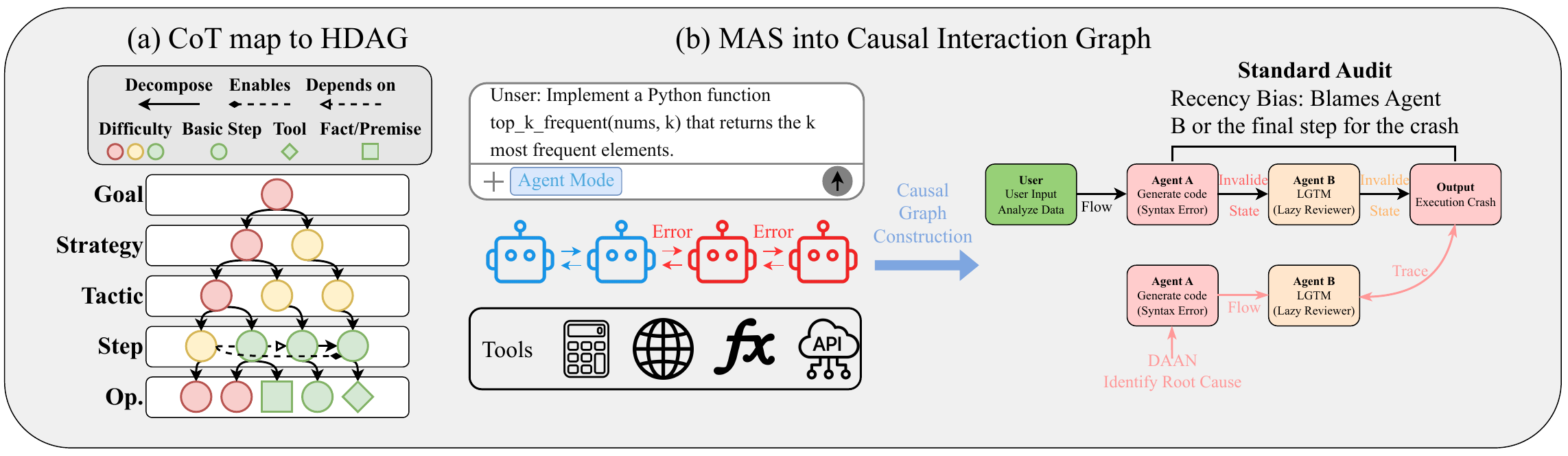}
    \caption{The dual decomposition engine of \TRUST{}. Linear Chain-of-Thought reasoning is mapped to Hierarchical Directed Acyclic Graphs (a), while Multi-Agent interactions are projected into Causal Interaction Graphs (b). Both representations enable parallel, privacy-preserving verification by distributed auditor networks.}
    \label{fig:decomposition}
\end{figure}

\subsection{Hierarchical Directed Acyclic Graphs (HDAGs)}\label{sec:hdag}

For singular Large Reasoning Models (LRMs) producing Chain-of-Thought (CoT) outputs~\citep{wei2022chain}, we introduce the Hierarchical Directed Acyclic Graph (HDAG). Linear text traces hide logical fallacies in prose; HDAGs expose them by enforcing structural rigor. This decomposition mirrors neural circuits in the frontal cortex, which process reasoning through multi-level evidence integration~\citep{morteza2019hierarchical}---just as the brain organizes reasoning hierarchically rather than linearly, our HDAG design enables different reasoning components to be audited at appropriate granularity.

Prior work on CoT decomposition, such as DLCoT~\citep{luo2025decon}, introduced automated frameworks for decomposing long reasoning traces into structured segments, primarily for model distillation. These works observe that CoTs can follow linear, tree, or more general network structures. DLCoT applies macro-structure parsing to divide CoTs into four parts---\emph{Problem Restatement}, \emph{Approach Exploration}, \emph{Verification}, and \emph{Summary}---before further segmenting into stepwise units. Other research focuses on extracting causal structures from token-level processing functions~\citep{kothapalli2025cot}.

In contrast, we propose a general, \textbf{problem-agnostic} approach applicable across mathematics, science, engineering, and open-domain reasoning.

\subsubsection{Formal Definition}

\begin{definition}[Hierarchical Directed Acyclic Graph]
An HDAG is defined as $G = (V, E, \mathcal{L})$, where:
\begin{itemize}
    \item $V = \{v_1, v_2, \ldots, v_n\}$ represents reasoning units (nodes), each carrying metadata: unique identifier, content summary, complexity score, assigned auditor type, and dependency list.
    \item $E \subseteq V \times V$ represents directed dependencies (edges) with semantic labels.
    \item $\mathcal{L} = \{l_{\text{goal}}, l_{\text{strategy}}, l_{\text{tactic}}, l_{\text{step}}, l_{\text{operation}}\}$ represents the five hierarchy levels.
\end{itemize}
\end{definition}

The hierarchy constraint requires that for any edge $(u, v) \in E$ where $\text{level}(u) = l_i$ and $\text{level}(v) = l_j$, we have $i \leq j$ (information flows downward through abstraction levels).

\subsubsection{The Five Abstraction Levels}

The decomposition engine parses raw model output into five abstraction layers, each with distinct verification requirements.

\textbf{Goal ($l_{\text{goal}}$).} The root node representing the user's initial prompt or problem statement. This level captures the high-level objective and success criteria, with verification focused on problem understanding and scope definition.

\textbf{Strategy ($l_{\text{strategy}}$).} High-level approaches proposed to solve the goal, such as ``Use integration by parts,'' ``Apply dynamic programming,'' or ``Decompose into subproblems.'' Strategy nodes require domain expertise to verify appropriateness and are typically assigned to Tier-2 or Tier-3 auditors.

\textbf{Tactic ($l_{\text{tactic}}$).} Specific methodologies implementing each strategy. A strategy of ``Use integration by parts'' might decompose into tactics such as ``Identify $u$ and $dv$'' and ``Apply the formula $\int u\,dv = uv - \int v\,du$.'' Tactics bridge abstract strategies and concrete execution.

\textbf{Step ($l_{\text{step}}$).} Discrete reasoning blocks that execute tactics. Each step represents a self-contained logical unit with clear inputs and outputs. Steps are the primary unit of semantic auditing, verified for logical consistency and faithfulness to the stated reasoning~\citep{lanham2023measuring}.

\textbf{Operation ($l_{\text{operation}}$).} Atomic calculations, code executions, or tool calls verifiable by deterministic checkers. Operations include arithmetic computations, syntax validation, unit test execution, and API call verification. These are assigned exclusively to Tier-1 Computational Auditors.

\begin{figure}[h]
    \centering
    \includegraphics[width=0.6\textwidth]{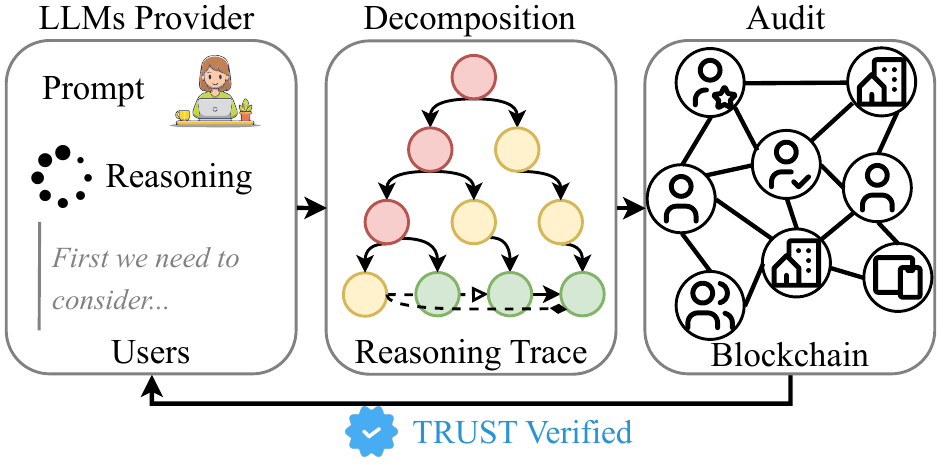}
    \caption{Example HDAG for a mathematical integration problem. Node color indicates the assigned difficulty level for different auditor tiers. Node shape denotes the type of step (basic reasoning, tool usage, or fact/premise). Edges represent logical relationships, including decomposition, dependency, and validation.}
    \label{fig:hdag}
\end{figure}

\subsubsection{Edge Semantics}

Edges in the HDAG carry semantic weight essential for auditing. We define five primary edge types.

\textbf{\textit{decomposes\_to}.} A parent node breaks down into children, as in Strategy $\to$ Tactics. This relationship is hierarchical and captures the refinement of abstract concepts into concrete actions.

\textbf{\textit{depends\_on}.} A sequential dependency where the validity of node $v_j$ requires the validity of node $v_i$, for example when Step 3 uses the result from Step 2. This edge type captures data flow and logical prerequisites.

\textbf{\textit{enables}.} A weaker form of dependency where node $v_i$ provides context or capability for $v_j$ without imposing a strict logical requirement, such as how importing a library enables subsequent function calls.

\textbf{\textit{validates}.} A self-correction or verification step that audits a previous sibling, signaled by phrases like ``Let me verify this calculation.'' These nodes are critical for detecting internal consistency and self-awareness in reasoning.

\textbf{\textit{contradicts}.} A detected logical inconsistency between nodes. When the decomposition engine or auditors identify contradictions, these edges are flagged for resolution, and traces with unresolved contradictions fail validation.
This edge semantics enables \TRUST{} to perform \textbf{dependency-aware auditing}: if a parent node fails validation, all dependent children are automatically marked as \textit{Invalid\_Cascade} rather than being evaluated independently, preventing wasted computation and false attributions.

\subsubsection{HDAG Construction Process}

\TRUST{} constructs HDAGs through a five-step automated process.

\textbf{Step 1: Identify Abstraction Levels.} The raw problem statement, reasoning trace (including any tool usage), and final output are parsed into the five semantic hierarchy levels using a trained classifier. We employ few-shot prompting with GPT-4 or fine-tuned models to identify level boundaries.

\textbf{Step 2: Segment Within Each Level.} Each level is further divided into granular units with associated metadata. Difficulty annotations (Easy/Medium/Hard) are assigned based on the presence of formal notation, domain specificity, the length and complexity of logical chains, and the availability of ground truth for verification.

\textbf{Step 3: Extract Relationships.} Logical dependencies between segments are mapped into the five edge types. We use dependency parsing, coreference resolution, and explicit markers (``therefore,'' ``using the above,'' ``let me verify'') to identify relationships.

\textbf{Step 4: Assign Auditor Types.} Segments are routed to auditor tiers based on level and difficulty. Operation-level nodes ($l_{\text{operation}}$) are handled by Tier-1 (Computational). Step-level nodes ($l_{\text{step}}$) are split by difficulty, with Easy and Medium steps handled by Tier-2 (LLM) and Hard steps escalated to Tier-3 (Human). Tactic, strategy, and goal nodes are assigned to Tier-2 or Tier-3 depending on domain expertise.

\textbf{Step 5: Validate and Finalize HDAG.} Quality assurance checks ensure that the graph is acyclic, all nodes are reachable from the goal, leaf nodes reside at $l_{\text{operation}}$ or $l_{\text{step}}$ levels, and no orphan nodes exist. The finalized HDAG is serialized for distribution to the auditing network.

\paragraph{Worked Example: HDAG Construction for Integration.}

Consider the problem ``Compute $\int x^2 e^x \, dx$.'' The reasoning trace decomposes as follows.

\textbf{Level 0 (Goal).} A single goal node $v_0$ states ``Compute $\int x^2 e^x \, dx$'' and is assigned to a Tier-2 auditor.

\textbf{Level 1 (Strategy).} Node $v_1$ proposes ``Apply integration by parts twice since $x^2$ requires two reductions'' (Tier-2), connected by $v_0 \xrightarrow{\text{decomposes\_to}} v_1$.

\textbf{Level 2 (Tactics).} The strategy refines into two tactic nodes, both at Tier-2: $v_2$ specifies the first application with $u = x^2$ and $dv = e^x dx$, while $v_3$ specifies the second application with $u = 2x$ and $dv = e^x dx$. The relevant edges are $v_1 \xrightarrow{\text{decomposes\_to}} v_2$ and $v_2 \xrightarrow{\text{enables}} v_3$.

\textbf{Level 3 (Steps).} Four step nodes execute the tactics. Node $v_4$ computes $du = 2x\,dx$ and $v = e^x$ (Tier-1); node $v_5$ applies the integration-by-parts formula to obtain $x^2 e^x - \int 2x e^x\,dx$ (Tier-2); node $v_6$ computes $du = 2\,dx$ and $v = e^x$ (Tier-1); and node $v_7$ applies the formula to obtain $2xe^x - \int 2e^x\,dx$ (Tier-2). Representative edges include $v_2 \xrightarrow{\text{decomposes\_to}} v_4$ and $v_4 \xrightarrow{\text{depends\_on}} v_5$.

\textbf{Level 4 (Operations).} Three operation nodes, all Tier-1, close out the computation. Node $v_8$ evaluates $\int 2e^x\,dx = 2e^x + C$ (verified by a computer algebra system), $v_9$ combines results to yield $x^2 e^x - 2xe^x + 2e^x + C$, and $v_{10}$ verifies the answer by symbolic differentiation.

This decomposition enables parallel verification: Tier-1 auditors can immediately verify $v_4$, $v_6$, $v_8$, $v_9$, and $v_{10}$ while Tier-2 auditors evaluate the strategy and step nodes. The $v_{10}$ node, connected through a \textit{validates} edge, triggers automatic verification against the original problem.

\subsection{Causal Interaction Graphs (CIGs) for Multi-Agent Systems}\label{sec:cig}

While HDAGs address single-model reasoning, Multi-Agent Systems (MAS) present a fundamentally different challenge: the ``Black Box of Black Boxes'' problem~\citep{cemri2025multifail}. When a swarm fails, the observed error is often far downstream of its cause, buried under layers of interaction and message passing. Recent empirical analyses reveal that LLM-based multi-agent systems fail at rates of 41\%-87\% on standard benchmarks, with failures clustering into three categories: specification issues, inter-agent misalignment, and inadequate task verification~\citep{cemri2025multifail}.

Current evaluation paradigms for MAS exhibit two critical failure modes:

\textbf{(C1) The Attribution Gap:} In collaborative swarms, downstream agents frequently inherit upstream faults. Linear audits exhibit \textit{recency bias}, blaming the last agent that touched the answer rather than the root cause. This mirrors the classic ``fault propagation'' problem in distributed systems~\citep{jha2024causalfault,atlas2024efficientfault}, where errors at one node cascade through dependencies, yet observability tools only surface symptoms at the terminal node.

\textbf{(C2) Protocol Opacity:} Many failures are not ``wrong reasoning'' but broken interfaces---format mismatches, ignored constraints, or semantic drift during transmission~\citep{cemri2025multifail,gradientinstitute2024riskanalysis}. These \textit{edge-level} violations between agents are invisible to output-only evaluators, which cannot distinguish whether an agent produced incorrect reasoning or simply received corrupted input.

To address these challenges, we integrate the \textbf{DAAN (Decentralized Audit and Active Refinement)} protocol into \TRUST{}, which projects interaction logs into \textbf{Causal Interaction Graphs (CIGs)}.

\subsubsection{Formal Definition}

\begin{definition}[Causal Interaction Graph]
A CIG is a directed graph $G_{\text{CIG}} = (A, M, C)$ where:
\begin{itemize}
    \item $A = \{a_1, a_2, \ldots, a_n\}$: The set of agents involved in the swarm, each with a defined role (e.g., Coder, Reviewer, Planner).
    \item $M = \{m_1, m_2, \ldots, m_k\}$: The set of messages passed between agents, each containing: sender, receiver, content, timestamp, and protocol metadata.
    \item $C \subseteq M \times M$: Causal links indicating trigger-response relationships between messages.
\end{itemize}
\end{definition}

Unlike a standard call graph, the CIG explicitly maps \textit{intent} against \textit{outcome}. If Agent A asks Agent B to ``Write a Python script for sorting,'' and Agent B responds with code, the edge $(m_{\text{request}} \to m_{\text{response}})$ represents both the intent (sorting script requested) and the outcome (code provided). The audit then verifies alignment between these.

\subsubsection{Node Projection}
Each node $v$ in the CIG encapsulates an atomic agent state with four components: a \textbf{Role} that designates the agent's function (e.g., Coder, Reviewer, Planner, Aggregator); an \textbf{Input} representing the context received from upstream agents or the environment; an \textbf{Output} representing the produced artifact (code, text, decision, or tool call result); and a \textbf{Status} that records the error classification assigned after audit (Valid, Invalid\_Root, Invalid\_Cascade, or Negligent). This formulation targets API-based systems where internal reasoning traces are not exposed; we audit the \textit{input-output transformation} rather than internal deliberation.

\subsubsection{Edge Mapping and Topological Properties}
Edges capture information flow between agents. An edge exists from node $v_i$ to $v_j$ if the output of $v_i$ appears in the input of $v_j$. Each edge carries three pieces of metadata: the \textbf{Transmitted Message}, which is the actual content passed between agents; a \textbf{Protocol Compliance Score} indicating whether the message conforms to the expected format and schema; and a \textbf{Fidelity Score} indicating whether key information was preserved without corruption or hallucination.

The CIG supports two critical operations. The first, $\texttt{get\_parents}(v)$, returns all nodes that sent input to $v$. The second, $\texttt{get\_descendants}(v)$, returns all downstream nodes affected if $v$ fails. A key property of the CIG is that it admits a \textbf{topological sort}, ordering nodes such that every agent is processed \textit{after} all of its upstream dependencies. This ordering enables deterministic fault localization by ensuring parent validity is known before classifying children.

\subsubsection{Dual-Layer Audit}
The \TRUST{} auditor performs verification at two distinct layers, separating \textit{content correctness} from \textit{transmission integrity}.

\paragraph{Layer 1: Node Validity Audit.}
The node validity oracle evaluates whether an agent's output is correct given its input and role. Without access to internal reasoning traces, we assess three dimensions. \textbf{Correctness} measures whether the output is factually and logically sound, with calculations verified, claims accurate, and no hallucinations introduced. \textbf{Completeness} measures whether the output addresses all requirements specified in the input. \textbf{Role Adherence} measures whether the output matches the agent's designated function. The auditor scores each node on a scale of $[0.0, 1.0]$, and nodes scoring below the threshold $\tau_{\text{node}} = 0.8$ are flagged as invalid. For mathematical tasks, this evaluation includes a \textbf{Segment Audit} that independently re-computes intermediate values. For code generation, we execute available unit tests.

\paragraph{Layer 2: Edge Integrity Audit (``The Telephone Test'').}
The edge audit validates information transmission through two sub-checks. The first sub-check, \textit{Protocol Compliance}, verifies that outputs conform to interface contracts; failures include format mismatches (e.g., JSON requested but prose returned), schema violations such as missing required fields or incorrect types, and ignored constraints where explicit instructions are disregarded. The second sub-check, \textit{Information Fidelity}, detects corruption during transmission across four failure modes: value corruption, in which key values such as numbers, names, or identifiers are altered; semantic drift, in which the core meaning is changed or distorted; information injection, in which the receiver hallucinates data not present in the message; and information loss, in which critical content is dropped by the receiver. Both checks apply a threshold of $\tau_{\text{edge}} = 0.8$, and edges that fail either check indicate transmission failure and require root cause analysis.

\subsubsection{Error Taxonomy and Fault Classification}

We assign each node to one of four mutually exclusive states based on its role in error propagation:

\begin{table}[h]
\centering
\caption{Error Taxonomy for CIG Nodes}
\label{tab:error_taxonomy}
\begin{tabular}{@{}lll@{}}
\toprule
\textbf{Status} & \textbf{Definition} & \textbf{Implication} \\
\midrule
\textsc{Valid} & Passes all verification checks & No action needed \\
\textsc{Invalid\_Root} & Incorrect output despite valid inputs & True failure point; requires repair \\
\textsc{Invalid\_Cascade} & Fails due to upstream invalid nodes & Inherited failure; auto-resolved \\
\textsc{Negligent} & Reviewer that approved invalid work & ``Lazy reviewer''; penalized \\
\bottomrule
\end{tabular}
\end{table}

The critical distinction: when a node fails a validity check, we examine its parents. If all parents are \textsc{Valid}, the node is classified as \textsc{Invalid\_Root}---it had good input but produced bad output. If any parent is invalid, the node is \textsc{Invalid\_Cascade}---its failure is inherited, not its own fault. This prevents downstream agents from being blamed for upstream mistakes.

\subsubsection{Deterministic Fault Localization Algorithm}

The provenance tracing algorithm processes nodes in topological order, ensuring parent status is known before classifying children:

\begin{algorithm}[H]
\caption{Deterministic Fault Localization}
\label{alg:fault_localization}
\begin{algorithmic}[1]
\REQUIRE Graph $G = (V, E)$, thresholds $\tau_{\text{node}}, \tau_{\text{edge}}$
\ENSURE Fault state $S(v)$ for all $v \in V$
\STATE $V_{\text{topo}} \gets \textsc{TopologicalSort}(G)$
\FOR{$v \in V_{\text{topo}}$}
    \STATE $P(v) \gets \textsc{GetParents}(v)$
    \STATE $\textit{propagated} \gets \exists p \in P(v): S(p) \neq \textsc{Valid}$
    \IF{any incoming edge fails protocol/fidelity check}
        \STATE $S(v) \gets \textsc{Invalid\_Root}$ \COMMENT{Protocol breach at boundary}
    \ELSIF{node validity score $\geq \tau_{\text{node}}$}
        \IF{$v$ is Reviewer role \AND approved invalid parent}
            \STATE $S(v) \gets \textsc{Negligent}$
        \ELSE
            \STATE $S(v) \gets \textsc{Valid}$
        \ENDIF
    \ELSE
        \IF{$\textit{propagated}$}
            \STATE $S(v) \gets \textsc{Invalid\_Cascade}$
        \ELSE
            \STATE $S(v) \gets \textsc{Invalid\_Root}$
        \ENDIF
    \ENDIF
\ENDFOR
\RETURN $\{S(v) : v \in V\}$
\end{algorithmic}
\end{algorithm}

\paragraph{Stationarity Detection for Cyclic Interactions.}
For cyclic interactions (e.g., Coder $\leftrightarrow$ Reviewer loops), we detect feedback traps by comparing consecutive outputs from the same agent. If semantic similarity exceeds the threshold $\tau_{\text{stat}} = 0.95$, the system is stuck in a loop and intervention is triggered to break the cycle with a divergence directive.

\subsection{Active Refinement: The Audit-Prune-Regenerate Loop}\label{sec:active_refinement}
Static auditing only identifies errors. \TRUST{} enables \textbf{Active Refinement}, a dynamic repair process that leverages the causal structure of HDAGs and CIGs for surgical correction. This approach is inspired by fault localization techniques in distributed systems~\citep{jha2024causalfault} and recent advances in targeted model refinement~\citep{lightman2023let}.

\subsubsection{The Prune-Freeze-Repair Cycle}
Upon identifying root cause nodes, \TRUST{} triggers surgical repair via a four-phase cycle.

\paragraph{Phase 1: Prune.}
Remove all \textsc{Invalid\_Root} and \textsc{Negligent} nodes from the graph, along with all of their downstream descendants (the nodes reachable via \texttt{get\_descendants}). Descendants must be re-run because they received tainted input; their outputs are causally contaminated regardless of their own apparent correctness.

\paragraph{Phase 2: Freeze.}
Retain all \textsc{Valid} nodes that do not depend on any failed node. Their computations are correct and do not need to be repeated. This is the key efficiency gain of the cycle: valid intermediate work is preserved rather than discarded.

\paragraph{Phase 3: Repair.}
Construct a feedback prompt containing the original input or context for the failed node, the audit critique enumerating the specific errors identified, the error type classification, and instructions for targeted correction. Re-prompt the model or agent with this feedback to generate a corrected output, after which the regenerated subgraph undergoes the full audit pipeline.

\paragraph{Phase 4: Re-Audit and Iterate.}
The new subgraph is submitted to the auditing network. This loop continues until one of three termination conditions is reached: all nodes pass validation, the maximum iteration depth is exceeded (default: 5 rounds), or stationarity is detected (no improvement between rounds).

\subsubsection{Cost Analysis}
The efficiency gains from active refinement depend on the location of errors within the graph topology.

\begin{proposition}[Repair Cost Bounds]
In a system with $N$ agents where an error occurs at depth $d$ in a tree of maximum depth $D$, the \textbf{global retry cost} is $O(N)$, since every node must be re-executed, whereas the \textbf{surgical repair cost} is $O(2^{D-d})$, covering only the failed node together with its descendants.
\end{proposition}

For leaf-level errors ($d = D$), surgical repair costs $O(1)$ versus $O(N)$ for global retry, an exponential improvement. Our experiments demonstrate this asymmetry across error depths: root errors at $d=1$ yield 0\% savings since the full graph must be regenerated, mid-level errors at $d=D/2$ yield 50\% to 75\% savings, and leaf errors at $d=D$ yield over 99\% savings.

\subsubsection{Feedback Injection Strategies}
The quality of regeneration depends critically on how audit feedback is presented to the model. We employ three strategies based on error type.

\textbf{Corrective Feedback.} For factual or logical errors, the prompt provides the specific mistake and the correct value:
\begin{quote}
``Your calculation of $\int 2e^x dx$ yielded $e^x$. The correct result is $2e^x + C$. Please recalculate the final answer.''
\end{quote}

\textbf{Directive Feedback.} For strategy or approach errors, the prompt provides guidance without revealing the solution:
\begin{quote}
``Your approach using substitution is inefficient for this integral. Consider integration by parts with $u = x^2$.''
\end{quote}

\textbf{Divergence Feedback.} For stationarity or loop detection, the prompt forces exploration of a new direction:
\begin{quote}
``Your previous two attempts produced identical outputs. You must try a fundamentally different approach. Consider: [alternative strategies].''
\end{quote}

A dynamic temperature schedule (initial $T = 0.7$, incrementing by $0.1$ per round) encourages exploration in later iterations when conservative approaches have failed.

\clearpage
\section{Network Design}\label{sec:network}

The \TRUST{} network architecture facilitates the auditing processes described above while maintaining decentralization, privacy, and efficiency. This section details the Decentralized Auditing Network (DAN), the multi-tier consensus mechanism, and the storage and privacy infrastructure.

\subsection{Decentralized Auditing Network (DAN)}\label{sec:dan}

The DAN is the workforce of the \TRUST{} platform, a permissionless network where entities stake tokens to register as auditors. Drawing on principles from Byzantine Fault Tolerant systems~\citep{castro1999practical,lamport2019byzantine}, the DAN achieves reliable consensus even when a fraction of participants are malicious or faulty.

\subsubsection{Three-Tier Auditor Architecture}

To optimize for both cost and accuracy, we classify auditors into three tiers based on capability and verification scope.

\paragraph{Tier 1: Computational Auditors (The ``Checkers'').}
Tier-1 auditors comprise deterministic algorithms, language interpreters (Python, JavaScript), theorem provers (Lean, Coq, Isabelle), computer algebra systems (SymPy, Mathematica), and static analyzers. They verify syntax correctness, execute code snippets in sandboxed environments, check arithmetic calculations, validate data formats, and confirm API response schemas. Their verification scope covers $l_{\text{operation}}$ nodes in HDAGs and edge protocol compliance in CIGs. By construction, the tier achieves an error rate of $\epsilon_C = 0$. Each verification costs between \$0.001 and \$0.01 with instant response under 100~ms, and the staking requirement is low, reflecting primarily computational resources. Tier-1 auditors handle the majority of verification volume, typically 60\% to 70\% of nodes, at negligible cost, enabling the system to scale efficiently.

\paragraph{Tier 2: Model-Based Auditors (The ``Evaluators'').}
Tier-2 auditors are specialized language models fine-tuned on auditing datasets, including Llama-3-8B-Audit, Mistral-7B-Verify, and custom models trained on process supervision data~\citep{lightman2023let,bai2022training}. Their role is to verify logical consistency, semantic coherence, standard reasoning steps, and moderate-complexity domain knowledge, as well as to assess the faithfulness of reasoning to stated conclusions~\citep{lanham2023measuring}. Their verification scope covers $l_{\text{step}}$ and $l_{\text{tactic}}$ nodes with Easy or Medium difficulty, along with node validity in CIGs for standard agent interactions. Empirical error rates are approximately $\epsilon_L \approx 0.05$ on standard reasoning tasks. Each verification costs between \$0.01 and \$0.10 with response times of one to five seconds, and the staking requirement is medium, combining GPU resources with a token stake. Tier-2 auditors provide the semantic understanding necessary for evaluating reasoning quality while remaining computationally tractable for high-volume verification.

\paragraph{Tier 3: Human-in-the-Loop (The ``Experts'').}
Tier-3 auditors are verified human domain experts, including PhD mathematicians, medical professionals such as licensed physicians and pharmacists, legal experts including attorneys and paralegals, senior software developers, and certified financial analysts. They verify high-ambiguity strategy nodes, assess safety implications, evaluate domain-specific correctness requiring specialized training, and resolve disputes from lower tiers. Their verification scope covers $l_{\text{strategy}}$ and $l_{\text{goal}}$ nodes, Hard-difficulty $l_{\text{step}}$ nodes, node validity for high-stakes agent decisions, and appeal resolution. Empirical error rates are approximately $\epsilon_H \approx 0.30$, reflecting human variability and fatigue and mitigated by multi-expert consensus, while adversarial probability is bounded by $\rho_H \leq 0.10$ through reputation and slashing mechanisms. Each verification costs between \$1 and \$10 with response times ranging from minutes to hours, and the staking requirement is high, combining a significant token stake with credential verification. Human experts are reserved for cases where automated verification is insufficient, ensuring that critical decisions receive appropriate scrutiny while minimizing human bottlenecks.

\begin{table}[h]
\centering
\caption{Auditor Tier Comparison}
\label{tab:auditor_tiers}
\begin{tabular}{@{}lccccc@{}}
\toprule
\textbf{Tier} & \textbf{Error Rate} & \textbf{Adversarial} & \textbf{Cost} & \textbf{Latency} & \textbf{Volume} \\
\midrule
Tier-1 (Computational) & $\epsilon_C = 0$ & $\rho_C = 0$ & \$0.001--0.01 & $<100$ms & 60--70\% \\
Tier-2 (LLM) & $\epsilon_L \approx 0.05$ & $\rho_L = 0$ & \$0.01--0.10 & 1--5s & 25--35\% \\
Tier-3 (Human) & $\epsilon_H \approx 0.30$ & $\rho_H \leq 0.10$ & \$1--10 & min--hrs & 5--10\% \\
\bottomrule
\end{tabular}
\end{table}

\subsubsection{Dynamic Auditor Assignment}

Nodes are routed to appropriate auditor tiers based on a routing function:
\begin{equation}
\text{Tier}(v) = f(\text{level}(v), \text{difficulty}(v), \text{domain}(v), \text{stakes}(v)),
\end{equation}
where $\text{level}(v)$ denotes the HDAG abstraction level or CIG node type, $\text{difficulty}(v)$ is the estimated complexity score (Easy/Medium/Hard), $\text{domain}(v)$ identifies the subject area (math, code, medical, legal, general), and $\text{stakes}(v)$ measures the consequence severity if verification fails. For high-stakes domains such as medical diagnosis, legal advice, and financial decisions, the routing function upgrades ambiguous cases to higher tiers, prioritizing accuracy over cost.

\subsection{Auditing \& Consensus Mechanism}\label{sec:audit_consensus}

Consensus is achieved via a stake-weighted voting mechanism that aggregates results from atomic ``Seat'' level up to global ``Trace'' level. The mechanism draws on Byzantine Fault Tolerant protocols~\citep{castro1999practical} while adapting to the specific requirements of AI auditing.

\subsubsection{Phase 1: Seat-Level Voting with Commit-Reveal}

For a given graph node $v$ requiring $k_t$ auditors of type $t$, the protocol proceeds as follows.

\paragraph{Committee Selection.}
The protocol uses a Verifiable Random Function (VRF) to select a committee $\mathcal{A}_v = \{a_1, \ldots, a_{k_t}\}$ based on stake weight:
\begin{equation}
P(a_i \in \mathcal{A}_v) \propto S_i \cdot \mathbb{I}[\text{type}(a_i) = t] \cdot \mathbb{I}[\text{available}(a_i)],
\end{equation}
where $S_i$ is auditor $i$'s staked tokens. Higher stake increases selection probability, aligning incentives with network security.

\paragraph{Commit-Reveal Scheme.}
To prevent ``herding'' (where auditors copy the majority vote) and to ensure independent evaluation, we employ a cryptographic commit-reveal protocol with three stages. In the \textbf{commit phase}, auditor $a_i$ generates vote $v_i \in \{\textsc{Pass}, \textsc{Fail}\}$ and random salt $r_i$, then submits commitment $c_i = \text{Hash}(v_i \| r_i)$ to the blockchain within time window $T_{\text{commit}}$. In the \textbf{reveal phase}, auditors submit $(v_i, r_i)$ within time window $T_{\text{reveal}}$ after the commit phase closes, and the smart contract verifies $\text{Hash}(v_i \| r_i) = c_i$. A \textbf{penalty for non-reveal} ensures that auditors who commit but fail to reveal forfeit their stake for that round, preventing strategic abstention. This scheme ensures that votes are cast independently before any auditor can observe others' decisions.

\subsubsection{Phase 2: Segment-Level Aggregation}

The smart contract aggregates revealed votes using stake-weighted threshold voting. Let $w_i = S_i / \sum_j S_j$ be the normalized stake of auditor $i$. Node $v$ passes validation if
\begin{equation}
\sum_{i=1}^{k_t} w_i \cdot \mathbb{I}(v_i = \textsc{Pass}) \geq \tau,
\end{equation}
where $\tau = 0.66$ is the consensus threshold (a two-thirds supermajority). This threshold ensures Byzantine fault tolerance: the system produces correct results even when up to $\lfloor (1-\tau) \cdot k_t \rfloor$ auditors are faulty or malicious.

\paragraph{Exact Pass Probability.}
For a segment of type $t$ with $k_t$ seats, error rate $\epsilon_t$, and adversarial fraction $\rho_t$, the exact pass probability is
\begin{equation}
p_t = \sum_{m=0}^{k_t} \binom{k_t}{m} \rho_t^m (1-\rho_t)^{k_t-m} \left[ \sum_{c=q_t}^{k_t-m} \binom{k_t-m}{c} (1-\epsilon_t)^c \epsilon_t^{k_t-m-c} \right],
\end{equation}
where $m$ malicious seats vote incorrectly, $c$ denotes the number of correct votes among the $k_t - m$ honest seats, and the quorum is $q_t = \lceil \tau \cdot k_t \rceil$.

\subsubsection{Phase 3: Trace-Level Validity}

The final validity of a reasoning trace or agent execution is determined by aggregating segment outcomes. Define the weighted pass total
\begin{equation}
W = \sum_{s=1}^{S} w_{t(s)} \cdot B_s,
\end{equation}
where $B_s \in \{0, 1\}$ is the segment pass indicator and $w_{t(s)}$ weights segments by type or importance. The trace passes if $W \geq W_\beta$, where
\begin{equation}
W_\beta = \beta \sum_{s=1}^{S} w_{t(s)},
\end{equation}
and $\beta \in (0, 1)$ is the trace-level quorum threshold.

A trace is valid if and only if three conditions hold simultaneously: all nodes on the \textbf{Critical Path} (the logic chain from goal to final answer) pass validation, no unresolved \textbf{contradicts} edges exist in the graph, and no \textsc{Invalid\_Root} nodes remain after active refinement attempts.

\subsubsection{Three-Layer Consensus Analysis}

We analyze consensus correctness at three layers.

\paragraph{Seat Layer.}
Within segment $s$, the $k_{t(s)}$ seats vote independently. Computer seats are noiseless ($\epsilon_C = 0$), LLM seats have error rate $\epsilon_L \approx 0.05$, and human seats have error rate $\epsilon_H \approx 0.30$ with adversarial probability $\rho_H \leq 0.10$.

\paragraph{Segment Layer.}
The segment pass indicator $B_s = \mathbb{I}[\#\{\text{correct votes}\} \geq q_{t(s)}]$ follows a mixture distribution over honest and adversarial configurations.

\paragraph{Trace Layer.}
We bound the failure probability $\Pr[W < W_\beta]$ using concentration inequalities.

\begin{proposition}[Trace-Level Safety Bound]
For any trace-level quorum $\beta \in (0,1)$,
\begin{equation}
\Pr[W < W_\beta] \leq \exp\left[ -\frac{2(\mu_{\text{vote}} - W_\beta)^2}{\sigma_{\max}^2} \right] \wedge \min_{\lambda > 0} \exp\left( \lambda W_\beta + \sum_{s=1}^{S} \ln(p_s e^{-\lambda w_s} + 1 - p_s) \right),
\end{equation}
where $\mu_{\text{vote}} = \mathbb{E}[W]$ and $\sigma_{\max}^2 = \sum_s w_s^2$.
\end{proposition}

The first term is a Hoeffding bound, while the second is a Chernoff bound optimized over $\lambda$. Under typical parameters ($\epsilon_L = 0.05$, $\epsilon_H = 0.30$, $\rho_H = 0.10$), achieving accuracy above 99.99\% requires $\beta \in [0.55, 0.75]$.

\subsection{Reputation, Rewards, and Slashing}\label{sec:economics}

The economic layer ensures long-term sustainability by making honest participation profitable and malicious behavior costly.

\subsubsection{Reputation-Weighted Slashing}

Each human auditor seat $i$ maintains a reputation score $r_i(t) \in [0, 1]$, updated after every segment:
\begin{equation}
r_i(t+1) = (1 - \gamma) \cdot r_i(t) + \gamma \cdot \mathbb{I}[\text{vote correct}],
\end{equation}
where $\gamma \in (0, 1]$ controls adaptation speed (default $\gamma = 0.1$). On an \textit{incorrect} vote (one that disagrees with consensus), the seat is slashed with probability
\begin{equation}
p_{\text{slash}}(r) = p_{\min} + (p_{\max} - p_{\min})(1 - r),
\end{equation}
with $0 < p_{\min} < p_{\max} \leq 1$. Low-reputation seats face higher slashing risk, creating strong incentives for consistent accuracy.

\subsubsection{Per-Segment Payoff Structure}

Let $X_i \in \{-P, 0, R\}$ denote the net payoff for seat $i$ on one segment:
\begin{equation}
X_i = \begin{cases}
R & \text{correct vote (aligns with consensus)} \\
0 & \text{incorrect vote, not slashed} \\
-P & \text{incorrect vote, slashed}.
\end{cases}
\end{equation}
For an honest seat with error rate $\epsilon_H$ and reputation $r$, the expected payoff is
\begin{equation}
\mu_H(r) = (1 - \epsilon_H) R - \epsilon_H \cdot P \cdot p_{\text{slash}}(r).
\end{equation}

\subsubsection{Numerical Example}

With parameters $R = 6$, $P = 8$, $p_{\min} = 0.2$, $p_{\max} = 0.5$, and $\epsilon_H = 0.30$, an honest seat with high reputation earns $\mu_{\min} = 0.7 \times 6 - 0.3 \times 8 \times 0.5 = 3.0$ in expected profit per segment, whereas a malicious seat incurs $\mathbb{E}[X_{\text{mal}}] = -0.5 \times 8 = -4.0$ in expected loss per segment. Over a 24-hour window with rate $\lambda = 60$ segments per hour ($T = 24$, $N_T = 1440$ segments), the probability of an honest auditor ending with non-positive payoff is below $10^{-80}$, while the probability of a malicious auditor breaking even is below $10^{-27}$. These bounds derive from the Safety-Profitability Theorem (Section~\ref{sec:theory}).

\subsection{Storage and Privacy Preservation}\label{sec:storage}

A critical requirement for enterprise adoption is data privacy. \TRUST{} ensures that no single auditor sees the full proprietary prompt or model logic through a privacy-by-design architecture.

\subsubsection{Off-Chain Storage with IPFS}

The full HDAG and CIG content is stored on the InterPlanetary File System (IPFS), a decentralized content-addressed storage network. Each node and segment is stored under a unique \textbf{content identifier} (CID) derived from its cryptographic hash. Content is \textbf{encrypted} using AES-256-GCM before upload, with decryption keys held only by authorized parties. \textbf{Redundancy} is achieved by pinning critical content across multiple IPFS nodes, ensuring availability and fault tolerance.

\subsubsection{Granular Access Control}

The key innovation enabling privacy-preserving distributed auditing is per-node encryption. Each node $v$ in the graph is encrypted with a unique symmetric key $K_v$, and auditors assigned to node $v$ receive only $K_v$ together with keys for immediate parents that provide the necessary context. As a result, no single auditor possesses keys sufficient to reconstruct the complete graph. Key distribution itself uses threshold cryptography, where $k$-of-$n$ schemes ensure that keys cannot be compromised by colluding minorities. This need-to-know structure ensures that proprietary reasoning logic is never exposed to any single party, preventing intellectual property theft and model distillation attacks~\citep{carlini2021extracting}.

\subsubsection{On-Chain Metadata}

The blockchain stores only lightweight metadata and cryptographic commitments. \textbf{IPFS CIDs} serve as pointers to encrypted off-chain content. A \textbf{Merkle root} provides a single hash summarizing the entire trace structure, enabling integrity verification. \textbf{Vote commitments} record hashed votes during the commit phase, while \textbf{final verdicts} record pass/fail status for each segment and the overall trace. \textbf{Auditor assignments} record which auditors verified which segments, supporting reward distribution and reputation updates. This separation ensures data integrity and public auditability without leaking sensitive content.

\subsubsection{Privacy Guarantees}

The \TRUST{} privacy architecture provides four guarantees. First, \textbf{segment isolation} ensures that individual auditors see only assigned segments, never complete traces. Second, \textbf{source anonymization} batches traces from multiple providers together so that auditors cannot identify which provider submitted which trace. Third, \textbf{no full reconstruction} is possible: even colluding auditors, up to threshold $t < n/3$, cannot reconstruct complete proprietary workflows. Fourth, \textbf{audit trail transparency} keeps content private while making verification outcomes public, enabling accountability without disclosure.

\subsection{Smart Contract Architecture}\label{sec:contracts}

The \TRUST{} protocol is implemented as a set of interoperating smart contracts on an EVM-compatible blockchain.

\begin{lstlisting}[language=Solidity, caption={Core TRUST Audit Interface}, label={lst:contract}]
interface ITRUSTAudit {
    // Session Management
    function createSession(bytes32 traceRoot, uint256 segmentCount) 
        external payable returns (uint256 sessionId);
    function assignAuditors(uint256 sessionId, uint256 segmentId, 
        address[] calldata auditors) external;
    
    // Commit-Reveal Voting
    function commitVote(uint256 sessionId, uint256 segmentId, 
        bytes32 commitment) external;
    function revealVote(uint256 sessionId, uint256 segmentId, 
        bool vote, bytes32 salt) external;
    
    // Consensus & Finalization
    function finalizeSegment(uint256 sessionId, uint256 segmentId) 
        external returns (bool passed);
    function finalizeTrace(uint256 sessionId) 
        external returns (bool valid);
    
    // Rewards & Slashing
    function distributeRewards(uint256 sessionId) external;
    function slashAuditor(address auditor, uint256 amount, 
        string calldata reason) external;
    
    // Events
    event SessionCreated(uint256 indexed sessionId, bytes32 traceRoot);
    event VoteCommitted(uint256 indexed sessionId, uint256 segmentId, 
        address auditor);
    event VoteRevealed(uint256 indexed sessionId, uint256 segmentId, 
        address auditor, bool vote);
    event SegmentFinalized(uint256 indexed sessionId, uint256 segmentId, 
        bool passed);
    event TraceFinalized(uint256 indexed sessionId, bool valid);
    event AuditorSlashed(address indexed auditor, uint256 amount);
}
\end{lstlisting}

The contract architecture supports three deployment properties. \textbf{Upgradability} is provided through proxy patterns that enable protocol improvements without disrupting ongoing sessions. \textbf{Gas optimization} is achieved through batch operations and Merkle proofs that minimize on-chain costs. \textbf{Cross-chain compatibility} is enabled by abstract interfaces that allow deployment on multiple L1 and L2 networks.

\subsection{Latency Analysis}\label{sec:latency}

End-to-end audit latency is critical for practical deployment. Table~\ref{tab:latency} decomposes the \TRUST{} pipeline.

\begin{table}[h]
\centering
\caption{Latency Breakdown for Typical Audit Session}
\label{tab:latency}
\begin{tabular}{@{}lcc@{}}
\toprule
\textbf{Component} & \textbf{Latency} & \textbf{Parallelizable} \\
\midrule
HDAG/CIG Construction & 2.1s & No \\
Auditor Assignment (VRF) & 0.3s & No \\
Content Distribution (IPFS) & 1.5s & Partial \\
Tier-1 Verification & 0.1s & Yes \\
Tier-2 Verification & 3.0s & Yes \\
Tier-3 Verification & 60--300s & Yes \\
Vote Aggregation & 0.2s & No \\
Consensus Finalization & 0.4s & No \\
\midrule
\textbf{Total (Tier-1/2 only)} & \textbf{7.5s} & \\
\textbf{Total (with Tier-3)} & \textbf{60--300s} & \\
\bottomrule
\end{tabular}
\end{table}

For cases requiring only Tier-1 and Tier-2 auditors, which cover approximately 90\% of cases, total latency is under 10 seconds, comparable to centralized evaluation. Traces requiring human expert verification incur additional delay but represent a small fraction of total volume. Compared to centralized single-LLM evaluation at approximately 5.2 seconds, \TRUST{} adds modest overhead of about 2.3 seconds in exchange for Byzantine fault tolerance, bias mitigation, and cryptographic auditability.

\clearpage
\section{Theoretical Guarantees}\label{sec:theory}

This section provides rigorous theoretical foundations for the \TRUST{} framework, establishing guarantees for both statistical safety (correctness of audit outcomes) and economic sustainability (incentive compatibility). We analyze the consensus mechanism at three hierarchical layers and prove that under appropriate parameter choices, the system achieves target safety levels while ensuring honest auditors profit and malicious actors incur losses.

\subsection{Overview and Notation}

The theoretical analysis proceeds in two phases. The \textbf{statistical analysis} characterizes the probability that a trace passes or fails audit as a function of auditor error rates, adversarial fractions, and voting parameters. The \textbf{economic analysis} designs reward and slashing mechanisms that make honest participation profitable and malicious behavior unprofitable in expectation. Table~\ref{tab:notation} summarizes the key notation used throughout this section.

\begin{table}[htp!]
\centering\small
\renewcommand{\arraystretch}{1.15}
\caption{Summary of Notation for Theoretical Analysis}
\label{tab:notation}
\begin{tabular}{@{}ll@{}}
\toprule
\multicolumn{2}{@{}l}{\textbf{Indices \& Random Counts}} \\ \midrule
$S$                              & Total number of segments in a trace \\
$N_T\sim\mathrm{Poisson}(\lambda T)$ & Number of segments audited in horizon $[0,T]$ \\
$t(s)\in\{\mathrm{C},\mathrm{L},\mathrm{H}\}$ & Auditor type of segment $s$ (Computer, LLM, Human) \\[3pt]

\multicolumn{2}{@{}l}{\textbf{Per-Segment Vote Variables (Statistical Layer)}} \\ \midrule
$k_t$                            & Number of auditor seats of type $t$ per segment \\
$q_t=\lceil \tau k_t\rceil$      & Per-type quorum threshold ($\tau$: vote threshold, typically 0.66) \\
$\epsilon_t$                     & Error rate for auditor type $t$ ($\epsilon_{\mathrm{C}}=0$, $\epsilon_{\mathrm{L}}\approx 0.05$, $\epsilon_{\mathrm{H}}\approx 0.30$) \\
$\rho_t$                         & Adversarial fraction for type $t$ ($\rho_{\mathrm{C}}=\rho_{\mathrm{L}}=0$, $\rho_{\mathrm{H}}\leq 0.10$) \\
$B_s\in\{0,1\}$                  & Segment pass indicator \\
$p_s=\Pr[B_s=1]$                 & Segment pass probability \\
$w_t$                            & Weight assigned to segment type $t$ \\[3pt]

\multicolumn{2}{@{}l}{\textbf{Trace-Level Aggregation}} \\ \midrule
$W=\sum_{s} w_{t(s)}B_s$         & Weighted pass total for one trace \\
$W_\beta=\beta\sum_{s} w_{t(s)}$ & Trace-level quorum threshold ($\beta\in(0,1)$) \\
$\mu_{\text{vote}}=\mathbb{E}[W]$ & Expected weighted pass total \\
$\sigma_{\text{vote}}^{2}$       & Variance of weighted pass total \\
$\sigma_{\max}^{2}=\sum_s w_s^2$ & Upper bound on variance \\[3pt]

\multicolumn{2}{@{}l}{\textbf{Economic Variables}} \\ \midrule
$R$                              & Reward for a correct vote \\
$P$                              & Penalty (slash amount) for incorrect vote \\
$r_i(t)\in[0,1]$                 & Reputation score for auditor $i$ at time $t$ \\
$\gamma\in(0,1]$                 & Reputation update rate \\
$p_{\text{slash}}(r)$            & Slashing probability given reputation $r$ \\
$p_{\min}, p_{\max}$             & Minimum and maximum slashing probabilities \\
$X_i\in\{-P,0,R\}$               & Net payoff for seat $i$ on one segment \\
$\mu_{\mathrm{H}}(r)$            & Expected payoff for honest seat with reputation $r$ \\
$\mu_{\min}$                     & Minimum expected payoff across all reputation levels \\
$\delta\in(0,1)$                 & Design constant: minimum loss per malicious seat \\[3pt]

\multicolumn{2}{@{}l}{\textbf{Cumulative Payoffs}} \\ \midrule
$U_{\mathrm{hon}}(T)$            & Total payoff for honest seat in horizon $[0,T]$ \\
$U_{\mathrm{mal}}(T)$            & Total payoff for malicious seat in horizon $[0,T]$ \\
$\lambda$                        & Segment arrival intensity (segments per unit time) \\
$T$                              & Time horizon \\
\bottomrule
\end{tabular}
\end{table}

\subsection{Three-Layer Consensus Analysis}\label{sec:consensus_theory}

The \TRUST{} consensus mechanism operates at three hierarchical layers: individual auditor votes (the seat layer), segment-level aggregation, and trace-level validation. We analyze each layer to derive the overall system reliability.

\subsubsection{Seat Layer: Independent Voting}

Within each segment $s$, the $k_{t(s)}$ auditor seats vote independently. The voting behavior depends on the auditor type. \textbf{Tier-1 (Computational)} auditors perform deterministic verification with $\epsilon_{\mathrm{C}} = 0$ and $\rho_{\mathrm{C}} = 0$, so they always vote correctly. \textbf{Tier-2 (LLM)} auditors perform probabilistic verification with error rate $\epsilon_{\mathrm{L}} \approx 0.05$ and $\rho_{\mathrm{L}} = 0$, since we assume no adversarial behavior for automated systems. \textbf{Tier-3 (Human)} auditors perform probabilistic verification with error rate $\epsilon_{\mathrm{H}} \approx 0.30$ and adversarial fraction $\rho_{\mathrm{H}} \leq 0.10$, where adversarial seats always vote incorrectly. For a given segment, each honest seat votes correctly with probability $1 - \epsilon_t$, while each adversarial seat votes incorrectly with probability 1.

\subsubsection{Segment Layer: Quorum-Based Aggregation}

Define the \textit{segment pass indicator}:
\begin{equation}
B_s = \mathbf{1}\left[\#\{\text{correct votes}\} \geq q_{t(s)}\right],
\end{equation}
where $q_t = \lceil \tau \cdot k_t \rceil$ is the quorum threshold for auditor type $t$, and $\tau = 0.66$ (a two-thirds supermajority) ensures Byzantine fault tolerance.

\begin{lemma}[Exact Segment Pass Probability]\label{lem:exact_type}
For a segment of type $t$ with parameters $(k_t, \epsilon_t, \rho_t)$ where $\rho_{\mathrm{C}} = \rho_{\mathrm{L}} = 0$, the exact pass probability is
\begin{equation}
p_t = \Pr[B_s = 1] = \sum_{m=0}^{k_t} \binom{k_t}{m} \rho_t^m (1-\rho_t)^{k_t-m} \left[ \sum_{c=q_t}^{k_t-m} \binom{k_t-m}{c} (1-\epsilon_t)^c \epsilon_t^{k_t-m-c} \right],
\end{equation}
where $m$ is the number of malicious seats (who always vote incorrectly), and $c$ is the number of correct votes among the $k_t - m$ honest seats.
\end{lemma}

\begin{proof}
The derivation follows from the law of total probability. We first condition on the number of malicious seats, $m \sim \text{Binomial}(k_t, \rho_t)$. Given $m$ malicious seats, the remaining $k_t - m$ honest seats each vote correctly independently with probability $1 - \epsilon_t$, so the number of correct votes among honest seats follows $\text{Binomial}(k_t - m, 1 - \epsilon_t)$. The segment passes if at least $q_t$ votes are correct, which requires $c \geq q_t$ correct votes from honest seats since malicious seats contribute zero correct votes. Summing over all configurations yields the result.
\end{proof}

\paragraph{Numerical Example.} For Tier-3 human auditors with $k_{\mathrm{H}} = 5$, $\epsilon_{\mathrm{H}} = 0.30$, $\rho_{\mathrm{H}} = 0.10$, and $\tau = 0.66$ (so $q_{\mathrm{H}} = 4$), the exact pass probability evaluates to $p_{\mathrm{H}} \approx 0.847$. This means approximately 84.7\% of human-audited segments are correctly passed, even with 10\% adversarial participation and a 30\% honest error rate.

\subsubsection{Trace Layer: Weighted Aggregation}

At the trace level, we aggregate segment outcomes using type-dependent weights. Define the weighted pass total $W = \sum_{s=1}^{S} w_{t(s)} B_s$ and the trace-level threshold $W_\beta = \beta \sum_{s=1}^{S} w_{t(s)}$, where $\beta \in (0,1)$ is the trace-level quorum parameter. With $B_s \sim \text{Bernoulli}(p_s)$ independent across segments, the first two moments are
\begin{align}
\mu_{\text{vote}} &= \mathbb{E}[W] = \sum_{s=1}^{S} w_s p_s, \label{eq:mu_vote} \\
\sigma_{\text{vote}}^{2} &= \text{Var}(W) = \sum_{s=1}^{S} w_s^{2} p_s(1-p_s) \leq \sum_{s=1}^{S} w_s^{2} \eqqcolon \sigma_{\max}^{2}. \label{eq:sigma_vote}
\end{align}
A trace passes validation if and only if $W \geq W_\beta$.

\begin{proposition}[Hoeffding and Chernoff Bounds for Trace Failure]\label{prop:trace_bounds}
For any trace-level quorum threshold $\beta \in (0,1)$, the probability that a valid trace fails audit is bounded by
\begin{equation}
\Pr[W < W_\beta] \leq \underbrace{\exp\left[-\frac{2(\mu_{\text{vote}} - W_\beta)^2}{\sigma_{\max}^2}\right]}_{\text{Hoeffding bound}} \wedge \underbrace{\min_{\lambda > 0} \exp\left(\lambda W_\beta + \sum_{s=1}^{S} \ln\left(p_s e^{-\lambda w_s} + (1-p_s)\right)\right)}_{\text{Chernoff bound}},
\end{equation}
where $a \wedge b = \min(a, b)$.
\end{proposition}

\begin{proof}
We prove each bound separately.

\textit{Hoeffding bound.} Each summand $X_s := w_s B_s$ satisfies $0 \leq X_s \leq w_s$. Applying Hoeffding's inequality to the sum of bounded independent random variables,
\begin{equation}
\Pr\left[\sum_s X_s < \mathbb{E}\left[\sum_s X_s\right] - a\right] \leq \exp\left(-\frac{2a^2}{\sum_s w_s^2}\right).
\end{equation}
Setting $a = \mu_{\text{vote}} - W_\beta$ yields the Hoeffding bound.

\textit{Chernoff bound.} For any $\lambda > 0$, Markov's inequality gives
\begin{align}
\Pr[W < W_\beta] &= \Pr[e^{-\lambda W} > e^{-\lambda W_\beta}] \\
&\leq e^{\lambda W_\beta} \mathbb{E}[e^{-\lambda W}] \\
&= e^{\lambda W_\beta} \prod_{s=1}^{S} \mathbb{E}[e^{-\lambda w_s B_s}] \\
&= e^{\lambda W_\beta} \prod_{s=1}^{S} \left(p_s e^{-\lambda w_s} + (1-p_s)\right),
\end{align}
where the third equality uses the independence of the $B_s$. Minimizing over $\lambda > 0$ yields the tightest Chernoff bound.
\end{proof}

\begin{remark}[Practical Implications]
Under typical parameters ($\epsilon_{\mathrm{L}} = 0.05$, $\epsilon_{\mathrm{H}} = 0.30$, $\rho_{\mathrm{H}} = 0.10$), achieving accuracy above 99.99\% requires setting $\beta \in [0.55, 0.75]$. The Chernoff bound is generally tighter for small failure probabilities, while the Hoeffding bound provides a simpler closed-form expression useful for parameter selection.
\end{remark}

\subsection{Committee Size and Accuracy Scaling}\label{sec:committee_scaling}

A natural question is how audit accuracy improves as committee size increases. The following result shows exponential improvement.

\begin{proposition}[Accuracy Improvement with Committee Size]\label{prop:committee_size}
For a segment with $k$ auditors, each with error rate $\epsilon < 0.5$ and no adversarial participation ($\rho = 0$), the probability of incorrect consensus under majority voting is bounded by
\begin{equation}
\Pr[\text{incorrect}] \leq \exp\left(-2k(0.5 - \epsilon)^2\right).
\end{equation}
\end{proposition}

\begin{proof}
Let $C$ be the number of correct votes, so that $C \sim \text{Binomial}(k, 1-\epsilon)$. The consensus is incorrect if $C < k/2$, equivalently $C < k(1-\epsilon) - k(0.5 - \epsilon)$. By Hoeffding's inequality,
\begin{equation}
\Pr[C < k/2] = \Pr[C < k(1-\epsilon) - k(0.5-\epsilon)] \leq \exp(-2k(0.5-\epsilon)^2). \qedhere
\end{equation}
\end{proof}

\paragraph{Numerical Example.} With $\epsilon = 0.30$ and $k = 7$ auditors,
\begin{equation}
\Pr[\text{incorrect}] \leq \exp(-2 \times 7 \times 0.04) = e^{-0.56} \approx 0.57,
\end{equation}
while increasing to $k = 21$ auditors gives
\begin{equation}
\Pr[\text{incorrect}] \leq \exp(-2 \times 21 \times 0.04) = e^{-1.68} \approx 0.19.
\end{equation}
This demonstrates how larger committees improve reliability, though with diminishing returns.

\subsection{Economic Layer: Incentive Design}\label{sec:econ_theory}

The economic layer ensures long-term sustainability by making honest participation profitable and malicious behavior costly. We design a reputation-weighted slashing mechanism that achieves this goal.

\subsubsection{Reputation Dynamics}

Each human auditor seat $i$ maintains a reputation score $r_i(t) \in [0,1]$, updated after each segment via an exponential moving average:
\begin{equation}\label{eq:reputation_update}
r_i(t+1) = (1-\gamma) r_i(t) + \gamma \cdot \mathbf{1}[\text{vote correct}],
\end{equation}
where $\gamma \in (0,1]$ controls the adaptation rate. Higher $\gamma$ makes reputation more responsive to recent performance, while lower $\gamma$ provides more stability.

\subsubsection{Reputation-Weighted Slashing}

On an \textit{incorrect} vote (one that disagrees with consensus), the auditor is slashed with probability
\begin{equation}\label{eq:slash_prob}
p_{\text{slash}}(r) = p_{\min} + (p_{\max} - p_{\min})(1 - r),
\end{equation}
where $0 < p_{\min} < p_{\max} \leq 1$. This design has three properties. Low-reputation auditors face higher slashing risk, creating strong incentives for consistent accuracy. Even high-reputation auditors face non-zero slashing probability $p_{\min}$, preventing complacency. The linear interpolation provides smooth incentive gradients across reputation levels.

\subsubsection{Per-Segment Payoff Structure}

Let $X_i \in \{-P, 0, R\}$ denote the net payoff for seat $i$ on one segment:
\begin{equation}
X_i = \begin{cases}
R & \text{correct vote (agrees with consensus)}, \\
0 & \text{incorrect vote, not slashed (probability } 1 - p_{\text{slash}}(r)), \\
-P & \text{incorrect vote, slashed (probability } p_{\text{slash}}(r)).
\end{cases}
\end{equation}
For an honest auditor with error rate $\epsilon_{\mathrm{H}}$ and reputation $r$, the expected per-segment payoff is
\begin{equation}\label{eq:expected_payoff}
\mu_{\mathrm{H}}(r) = \mathbb{E}[X_i] = (1 - \epsilon_{\mathrm{H}}) R - \epsilon_{\mathrm{H}} \cdot P \cdot p_{\text{slash}}(r),
\end{equation}
and the minimum expected payoff (achieved at lowest reputation, $r = 0$) is
\begin{equation}\label{eq:mu_min}
\mu_{\min} = (1 - \epsilon_{\mathrm{H}}) R - \epsilon_{\mathrm{H}} \cdot P \cdot p_{\max}.
\end{equation}

\subsubsection{Variance Bounds for MGF Analysis}

For the moment-generating function (MGF) analysis required in our main theorem, we need bounds on the range and variance of payoffs.

\paragraph{Range Bound.} The centered increment $Y_i = X_i - \mathbb{E}[X_i]$ satisfies $Y_i \leq b$, where
\begin{equation}\label{eq:range_bound}
b := R.
\end{equation}
This conservative bound preserves the validity of the MGF domain.

\paragraph{Variance Bound.} The variance of $X_i$ is bounded by
\begin{equation}\label{eq:variance_bound}
\sigma_{\mathrm{H}}^2 := \sup_{r \in [0,1]} \text{Var}[X_i] = \sup_{r \in [0,1]} \left[(1-\epsilon_{\mathrm{H}}) R^2 + \epsilon_{\mathrm{H}} p_{\text{slash}}(r) P^2 - \mu_{\mathrm{H}}(r)^2\right].
\end{equation}
A simpler upper bound is $\sigma_{\mathrm{H}}^2 \leq (R + P)^2 / 4$, which is achieved when payoffs are symmetric around zero.

\subsection{Auxiliary Lemmas}

Before stating the main theorem, we establish a key lemma for bounding moment-generating functions.

\begin{lemma}[MGF Bound for Bounded Centered Random Variables]\label{lem:mgf_bound}
Let $W$ satisfy $\mathbb{E}[W] = 0$, $\mathbb{E}[W^2] = \sigma^2$, and $W \leq b$ almost surely with $b > 0$. Then for any $\theta \in (0, 3/b)$,
\begin{equation}
\mathbb{E}[e^{\theta W}] \leq \exp\left(\frac{\theta^2 \sigma^2}{2(1 - \theta b/3)}\right).
\end{equation}
\end{lemma}

\begin{proof}
This follows from the standard Bernstein-Bernoulli expansion. The key insight is that for bounded random variables, the MGF depends only on the variance and range. The factor $(1 - \theta b/3)^{-1}$ accounts for higher-order moments, becoming relevant when $\theta$ approaches $3/b$.
\end{proof}

\subsection{Main Theorem: Safety-Profitability Guarantee}\label{sec:main_theorem}

We now state and prove the central theoretical result of \TRUST{}, which establishes that appropriate parameter choices simultaneously achieve statistical safety and economic sustainability.

\begin{theorem}[Safety-Profitability Guarantee]\label{thm:safety_profit}
Fix a time horizon $T > 0$, a target trace-failure probability $\epsilon_{\text{target}} \in (0,1)$, and a design constant $\delta \in (0,1)$. The system has two ``dials'' for controlling safety and profitability.

\paragraph{Statistical Dial (S1).} Let $(k_t, q_t, w_t, \beta)$ be the voting parameters. Write $\mu_{\text{vote}} := \mathbb{E}[W]$ and $\sigma_{\text{vote}}^2 := \sup \text{Var}(W)$ for one trace. Require
\begin{equation}\tag{S1}\label{eq:S1}
\mu_{\text{vote}} - W_\beta \geq \sqrt{\frac{1}{2} \sigma_{\text{vote}}^2 \ln\frac{\lambda T}{\epsilon_{\text{target}}}}.
\end{equation}

\paragraph{Economic Dial (E1, E2).} Choose $(R, P, p_{\min}, p_{\max})$ such that
\begin{align}
R &> \frac{\epsilon_{\mathrm{H}}}{1 - \epsilon_{\mathrm{H}}} \cdot P \cdot p_{\max}, \tag{E1}\label{eq:E1} \\
p_{\min} &\geq \frac{\delta}{1 - \alpha}, \quad \text{where } \alpha := \frac{P \cdot p_{\max}}{R + P \cdot p_{\max}}. \tag{E2}\label{eq:E2}
\end{align}

Define the minimum expected earnings per segment as
\begin{equation}
\mu_{\min} := (1 - \epsilon_{\mathrm{H}}) R - \epsilon_{\mathrm{H}} P \cdot p_{\max} > 0.
\end{equation}

Then the following guarantees hold.

\textnormal{\textbf{(a) Statistical Safety.}}
\begin{equation}
\Pr[\text{trace fails in } [0,T]] \leq \epsilon_{\text{target}}.
\end{equation}

\textnormal{\textbf{(b) Honest Profitability.}}
\begin{equation}
\Pr[U_{\mathrm{hon}}(T) \leq 0] \leq \exp\left[-\frac{\lambda T \mu_{\min}^2}{2\sigma_{\mathrm{H}}^2 + \frac{2}{3} b \mu_{\min}}\right].
\end{equation}

\textnormal{\textbf{(c) Malicious Loss.}}
\begin{equation}
\Pr[U_{\mathrm{mal}}(T) \geq 0] \leq \exp\left[-\frac{\lambda T (\delta P)^2}{2\sigma_{\mathrm{H}}^2 + \frac{2}{3} b \delta P}\right], \quad \mathbb{E}[U_{\mathrm{mal}}(T)] \leq -\lambda T \delta P.
\end{equation}
\end{theorem}

\begin{proof}
We prove each claim separately.

\paragraph{(a) Statistical Safety.}

A single trace's weighted pass sum $W$ satisfies $0 \leq W \leq \sum_t w_t$ with $\mathbb{E}[W] = \mu_{\text{vote}}$ and $\text{Var}(W) \leq \sigma_{\text{vote}}^2$. By Hoeffding's inequality (Proposition~\ref{prop:trace_bounds}), for any $a > 0$,
\begin{equation}
\Pr[W < \mu_{\text{vote}} - a] \leq \exp\left(-\frac{2a^2}{\sigma_{\text{vote}}^2}\right).
\end{equation}
Setting $a = \mu_{\text{vote}} - W_\beta$ and using condition~\eqref{eq:S1},
\begin{equation}
\frac{2(\mu_{\text{vote}} - W_\beta)^2}{\sigma_{\text{vote}}^2} \geq \ln\frac{\lambda T}{\epsilon_{\text{target}}},
\end{equation}
so that
\begin{equation}\label{eq:per_trace_fail}
p_{\text{trace-fail}} := \Pr[W < W_\beta] \leq \frac{\epsilon_{\text{target}}}{\lambda T}.
\end{equation}
Traces arrive according to a Poisson process with intensity $\lambda$, so $N_T \sim \text{Poisson}(\lambda T)$ traces are audited in $[0,T]$. By a union bound,
\begin{align}
\Pr[\text{at least one trace fails in } [0,T]] &= \Pr[\exists \text{ trace with } W < W_\beta] \\
&\leq \mathbb{E}[N_T] \cdot p_{\text{trace-fail}} \\
&= \lambda T \cdot \frac{\epsilon_{\text{target}}}{\lambda T} = \epsilon_{\text{target}}.
\end{align}
This completes the proof of (a).

\paragraph{(b) Honest Profitability.}

Let $U_{\mathrm{hon}}(T) = \sum_{i=1}^{N_T} X_i$ denote the cumulative payoff for an honest auditor over horizon $[0,T]$.

\textit{Step 1: Center and bound increments.} Define centered variables $Y_i := X_i - \mu_{\min}$. Then $\mathbb{E}[Y_i] \geq 0$, $Y_i \leq b := R$, and $\text{Var}(Y_i) \leq \sigma_{\mathrm{H}}^2$.

\textit{Step 2: MGF bound.} By Lemma~\ref{lem:mgf_bound}, for any $\theta \in (0, 3/b)$,
\begin{equation}\label{eq:mgf_yi}
\mathbb{E}[e^{\theta Y_i}] \leq \exp\left(\frac{\theta^2 \sigma_{\mathrm{H}}^2}{2(1 - \theta b/3)}\right).
\end{equation}

\textit{Step 3: Chernoff bound for random sum.} We have $U_{\mathrm{hon}}(T) = N_T \mu_{\min} + \sum_{j=1}^{N_T} Y_j$. Conditioning on $N_T = n$ and applying Chernoff's inequality,
\begin{align}
\Pr[U_{\mathrm{hon}}(T) \leq 0 \mid N_T = n] &= \Pr\left[\sum_{j=1}^n Y_j \leq -n\mu_{\min}\right] \\
&\leq e^{-\theta n \mu_{\min}} \left(\mathbb{E}[e^{\theta Y_1}]\right)^n.
\end{align}
Removing conditioning using the Poisson MGF $\mathbb{E}[z^{N_T}] = \exp(\lambda T(z-1))$,
\begin{equation}
\Pr[U_{\mathrm{hon}}(T) \leq 0] \leq \exp\left(\lambda T \left(e^{-\theta \mu_{\min}} \mathbb{E}[e^{\theta Y_1}] - 1\right)\right).
\end{equation}
Substituting the MGF bound~\eqref{eq:mgf_yi},
\begin{equation}\label{eq:honest_bound_theta}
\Pr[U_{\mathrm{hon}}(T) \leq 0] \leq \exp\left(\lambda T \left\{-\theta \mu_{\min} + \frac{\theta^2 \sigma_{\mathrm{H}}^2}{2(1-\theta b/3)}\right\}\right).
\end{equation}

\textit{Step 4: Optimize $\theta$.} Let $g(\theta) := -\theta \mu_{\min} + \frac{\theta^2 \sigma_{\mathrm{H}}^2}{2(1-\theta b/3)}$, and substitute $t := \theta b/3 \in (0,1)$:
\begin{equation}
g(t) = -\frac{3t\mu_{\min}}{b} + \frac{9t^2 \sigma_{\mathrm{H}}^2}{2b^2(1-t)}.
\end{equation}
Setting $g'(t) = 0$ yields
\begin{equation}
t^* = 1 - \frac{1}{\sqrt{1 + 2b\mu_{\min}/(3\sigma_{\mathrm{H}}^2)}},
\end{equation}
and substituting back gives
\begin{equation}
g(t^*) = -\frac{\mu_{\min}^2}{2\sigma_{\mathrm{H}}^2 + \frac{2}{3}b\mu_{\min}}.
\end{equation}

\textit{Step 5: Final bound.} Combining with~\eqref{eq:honest_bound_theta},
\begin{equation}
\Pr[U_{\mathrm{hon}}(T) \leq 0] \leq \exp\left(-\frac{\lambda T \mu_{\min}^2}{2\sigma_{\mathrm{H}}^2 + \frac{2}{3}b\mu_{\min}}\right).
\end{equation}
This completes the proof of (b).

\paragraph{(c) Malicious Loss.}

A malicious auditor always votes incorrectly. Conditions~\eqref{eq:E1} and~\eqref{eq:E2} ensure $\mathbb{E}[X_i] \leq -\delta P < 0$.

\textit{Step 1: Center increments.} Define $Z_i := X_i + \delta P$ so that $\mathbb{E}[Z_i] \geq 0$ and $Z_i \leq b$.

\textit{Step 2: Apply MGF analysis.} The analysis is identical to part (b), replacing $\mu_{\min}$ with $\delta P$:
\begin{equation}
\Pr[U_{\mathrm{mal}}(T) \geq 0] \leq \exp\left(-\frac{\lambda T (\delta P)^2}{2\sigma_{\mathrm{H}}^2 + \frac{2}{3}b\delta P}\right).
\end{equation}

\textit{Step 3: Expected loss.} By linearity of expectation,
\begin{equation}
\mathbb{E}[U_{\mathrm{mal}}(T)] = \mathbb{E}[N_T] \cdot \mathbb{E}[X_i] = \lambda T \cdot \mathbb{E}[X_i] \leq -\lambda T \delta P.
\end{equation}
This completes the proof of (c).
\end{proof}

\subsection{Numerical Calibration}\label{sec:numerical}

We illustrate the theorem with concrete parameter choices. We set the reward $R = 6$ and penalty $P = 8$, slashing probabilities $p_{\min} = 0.2$ and $p_{\max} = 0.5$, error rate $\epsilon_{\mathrm{H}} = 0.30$, design constant $\delta = 0.2$, and arrival rate $\lambda = 60$ segments per hour.

\paragraph{Verification of Conditions.} Condition~\eqref{eq:E1} is satisfied since $R = 6 > \tfrac{0.30}{0.70} \times 8 \times 0.5 = 1.71$. Computing $\alpha = \tfrac{8 \times 0.5}{6 + 8 \times 0.5} = 0.4$, condition~\eqref{eq:E2} requires $p_{\min} \geq \tfrac{0.2}{1 - 0.4} = 0.33$, which is not satisfied by $p_{\min} = 0.2$. Adjusting to $p_{\min} = 0.35$ satisfies all conditions, and we adopt this value below.

\paragraph{Expected Payoffs.} An honest auditor earns expected per-segment profit $\mu_{\min} = 0.70 \times 6 - 0.30 \times 8 \times 0.5 = 3.0$, while a malicious auditor incurs $\mathbb{E}[X_{\mathrm{mal}}] = -0.5 \times 8 = -4.0$ per segment.

\paragraph{Variance Calculation.}
\begin{equation}
\sigma_{\mathrm{H}}^2 = 0.70 \times 36 + 0.30 \times 0.5 \times 64 - 9 = 25.2 + 9.6 - 9 = 25.8.
\end{equation}

\paragraph{Tail Probabilities over 24 Hours.} With $T = 24$ hours and $\lambda T = 1440$ segments, the honest auditor loss probability satisfies
\begin{equation}
\Pr[U_{\mathrm{hon}}(24\text{h}) \leq 0] \leq \exp\left(-\frac{1440 \times 9}{2 \times 25.8 + \frac{2}{3} \times 6 \times 3}\right) = \exp\left(-\frac{12960}{63.6}\right) \approx e^{-204} < 10^{-88},
\end{equation}
and the malicious auditor profit probability satisfies
\begin{equation}
\Pr[U_{\mathrm{mal}}(24\text{h}) \geq 0] \leq \exp\left(-\frac{1440 \times (0.2 \times 8)^2}{2 \times 25.8 + \frac{2}{3} \times 6 \times 1.6}\right) = \exp\left(-\frac{3686.4}{58}\right) \approx e^{-63.6} < 10^{-27}.
\end{equation}
These astronomically small probabilities demonstrate that an honest auditor is virtually guaranteed to profit over a 24-hour period, while a malicious auditor is virtually guaranteed to lose money over the same period.

\subsection{Parameter Selection Guidelines}\label{sec:param_guidelines}

Based on Theorem~\ref{thm:safety_profit}, we provide practical guidelines for deploying \TRUST{} in three phases.

\paragraph{Reliability Phase.} Select voting parameters $(k_t, q_t, w_t, \beta)$ such that condition~\eqref{eq:S1} holds with a safety margin (for example, twice the required gap), and use Monte Carlo simulation to verify performance under worst-case adversarial fractions.

\paragraph{Economics Phase.} Choose reward and penalty parameters $(R, P, p_{\min}, p_{\max})$ that satisfy conditions~\eqref{eq:E1} and~\eqref{eq:E2}. Three considerations guide the choice. A higher $R/P$ ratio makes the system more attractive to honest participants. A higher $p_{\min}$ provides stronger deterrence against malicious behavior. The overall ratio should account for the expected honest error rate $\epsilon_{\mathrm{H}}$.

\paragraph{Stress Testing Phase.} Conduct Monte Carlo simulations that vary adversarial fractions across $\rho_{\mathrm{H}} \in [0, 0.30]$, inflate error rates across $\epsilon_{\mathrm{H}} \in [0.20, 0.40]$, and explore different trace lengths and compositions. Iterate on parameters until the empirical failure probability meets $\epsilon_{\text{target}}$.

\subsection{Extensions and Robustness}\label{sec:extensions}

\subsubsection{Correlated Auditor Errors}

The analysis assumes independent auditor errors. In practice, errors may be correlated, for example when multiple LLMs make similar mistakes. To handle this case, the deployment can use auditors from diverse model families, introduce explicit correlation parameters in the segment pass probability, and apply concentration inequalities for weakly dependent random variables.

\subsubsection{Adaptive Adversaries}

The current model assumes adversaries vote incorrectly with a fixed probability. Against adaptive adversaries who strategically time their attacks, three mechanisms provide defense. The reputation system penalizes inconsistent behavior, naturally discouraging strategic attacks. The commit-reveal voting protocol prevents adversaries from conditioning their vote on others' decisions. Random auditor assignment via VRF further limits targeted attacks against specific traces or domains.

\subsubsection{Dynamic Error Rates}

When auditor error rates vary over time, due to fatigue or changing task difficulty, three properties of the system accommodate the change. The reputation update rule in~\eqref{eq:reputation_update} naturally tracks recent performance through its exponential moving average. Adaptive routing can assign easier tasks to lower-performing auditors. The variance bounds in Theorem~\ref{thm:safety_profit} accommodate worst-case scenarios across the full operating range of $\epsilon_{\mathrm{H}}$.

\subsection{Summary of Theoretical Guarantees}

The theoretical analysis establishes five guarantees for \TRUST{}.

\textbf{Provable safety.} Under condition~\eqref{eq:S1}, the probability of a valid trace failing audit is bounded by $\epsilon_{\text{target}}$, even with up to 30\% adversarial human auditors.

\textbf{Economic sustainability.} Under conditions~\eqref{eq:E1} and~\eqref{eq:E2}, honest auditors earn positive expected profit while malicious auditors incur expected losses.

\textbf{Exponential tail bounds.} Both the honest profit and malicious loss guarantees have exponentially decaying tail probabilities, ensuring high confidence over operational time horizons.

\textbf{Scalable accuracy.} Audit accuracy improves exponentially with committee size, providing a clear knob for trading off cost against reliability.

\textbf{Byzantine fault tolerance.} The two-thirds quorum threshold ensures correct outcomes even when up to one-third of auditors are faulty or malicious.

Together, these guarantees provide the theoretical foundation for deploying \TRUST{} in high-stakes applications where both correctness and economic viability are critical.

\clearpage
\section{Applications}\label{sec:applications}

The \TRUST{} framework extends beyond its core auditing capabilities to enable a comprehensive ecosystem of decentralized AI services. This section details four primary applications that address critical gaps in the current AI infrastructure landscape.

\subsection{Decentralized Auditing (A1)}\label{sec:app_audit}

The primary application of \TRUST{} is the verification of high-stakes reasoning in domains where errors carry significant consequences. Unlike output-only evaluation, \TRUST{} semantic auditing examines the complete reasoning process to identify flawed logic that may coincidentally produce correct answers.

\paragraph{Medical Diagnosis and Clinical Decision Support.}
Healthcare AI systems must demonstrate not only accuracy but also sound clinical reasoning to meet regulatory requirements and ensure patient safety. \TRUST{} audits the reasoning chains of medical chatbots and clinical decision support systems to verify adherence to established clinical guidelines.

\begin{tcolorbox}[metabox, title=\textbf{Medical Auditing Pipeline}]
\small
The pipeline operates in five sequential stages. \textbf{Trace Submission}: the clinical AI generates a reasoning trace for diagnosis or treatment recommendation. \textbf{HDAG Decomposition}: the trace is decomposed into clinical rule applications, such as the steps of CHA$_2$DS$_2$-VASc scoring. \textbf{Guideline Verification}: computational auditors verify calculations, while LLM auditors check evidence extraction from patient notes. \textbf{Human Review}: complex cases escalate to Tier-3 clinical experts for final verification. \textbf{Certification}: approved reasoning receives an on-chain attestation for regulatory compliance.
\end{tcolorbox}

As demonstrated in Section~\ref{sec:example}, this pipeline detects ``correct answer, wrong reason'' scenarios that output-only auditing misses entirely, preventing deployment of models with brittle reasoning that fails under distribution shift.

\paragraph{Legal Discovery and Document Analysis.}
Legal AI applications face strict requirements for accuracy and traceability. Hallucinated case citations or misrepresented precedents can result in sanctions, malpractice claims, and harm to clients. \TRUST{} provides verifiable audit trails for legal AI reasoning across three checks. \textbf{Citation verification} cross-references cited cases against legal databases such as Westlaw and LexisNexis to detect fabricated precedents, performed by computational auditors. \textbf{Argument tracing} verifies that conclusions follow logically from cited authorities, performed by LLM auditors. \textbf{Privilege detection} reviews flagged segments for potential privilege violations before disclosure, performed by human auditors.

\paragraph{Mathematical and Scientific Verification.}
Formal reasoning systems such as AlphaGeometry, AlphaProof, and mathematical LLMs generate proofs that require rigorous verification. \TRUST{} decomposes proofs into atomic inference steps, enabling parallel verification along three dimensions. \textbf{Proof step verification} independently checks each inference step against formal rules. \textbf{Dependency checking} uses the graph structure to ensure lemmas are proven before use in subsequent steps. \textbf{Completeness auditing} verifies that all cases are covered and no steps are skipped.

\subsection{Decentralized Model Leaderboard (A2)}\label{sec:app_model}

Current AI model leaderboards face fundamental integrity challenges. Centralized platforms such as LMSYS Chatbot Arena, while valuable, are susceptible to manipulation through selective release (companies submitting only high-performing model versions), vote manipulation, and benchmark contamination. \TRUST{} enables a blind, tamper-proof leaderboard architecture.

\paragraph{Threat Model.}
Existing leaderboards are vulnerable to four attack vectors. \textbf{Selective submission} occurs when companies run internal evaluations and only publicly submit versions that score well, creating misleading performance impressions. \textbf{Benchmark contamination} arises when training data includes benchmark questions, inflating scores without improving genuine capability. \textbf{Vote manipulation} targets crowdsourced preference data through coordinated voting campaigns. \textbf{Cherry-picking} retroactively filters results to show only favorable comparisons.

\paragraph{TRUST Leaderboard Architecture.}
The decentralized leaderboard eliminates these vulnerabilities through cryptographic commitments and distributed evaluation.

\begin{tcolorbox}[metabox, title=\textbf{Blind Evaluation Protocol}]
\small
\textbf{Phase 1: Test Set Commitment.} Evaluation sponsors submit encrypted test sets with hash commitments on-chain, test questions remain hidden until evaluation completes, and sponsors stake tokens as collateral against test set quality.

\textbf{Phase 2: Model Registration.} Model providers register model checkpoints with immutable hashes. Registration requires a stake deposit, and withdrawal triggers score invalidation. No model modifications are permitted after registration.

\textbf{Phase 3: Decentralized Inference.} Test questions are distributed to decentralized inference nodes, each question is processed by multiple independent nodes for consensus, and model outputs are recorded with cryptographic attestations.

\textbf{Phase 4: Blind Auditing.} DAN auditors evaluate outputs without knowledge of model identity, scores are aggregated through stakeholder-weighted consensus, and final scores are minted as non-fungible attestations on-chain.
\end{tcolorbox}

\paragraph{Score Immutability and Transparency.}
Once minted, leaderboard scores cannot be altered, deleted, or cherry-picked. The complete evaluation history remains permanently auditable through four properties. \textbf{Historical tracking} preserves all model versions and their scores permanently. \textbf{Contamination detection} uses statistical analysis to flag suspicious performance patterns suggesting benchmark contamination. \textbf{Cross-benchmark correlation} treats inconsistent performance across related benchmarks as a trigger for additional scrutiny. \textbf{Public verification} allows anyone to verify score calculations using on-chain data and open-source aggregation code.

\paragraph{Economic Incentives.}
The leaderboard creates aligned incentives across four participant classes. \textbf{Model providers} face a stake requirement that ensures commitment, with honest evaluation building reputation over time. \textbf{Test set sponsors} attract evaluation fees through quality test sets, while poor-quality submissions result in stake slashing. \textbf{Inference nodes} earn fees for computation and are slashed for inconsistent outputs. \textbf{Auditors} operate under the standard \TRUST{} incentive structure described in Section~\ref{sec:economics}.

\subsection{Decentralized Data Annotation (A3)}\label{sec:app_annot}

Reinforcement Learning from Human Feedback (RLHF) and related alignment techniques require massive quantities of high-quality preference data. Current annotation pipelines suffer from quality inconsistency, lack of verifiability, and centralized control over labeling standards. \TRUST{} establishes a trustless marketplace for ``Proof-of-Quality'' annotation data.

\paragraph{Current Annotation Challenges.}
The AI industry faces a data quality crisis along four dimensions. \textbf{Quality variance} appears in crowdsourced annotations as high inter-annotator disagreement and inconsistent standards. \textbf{Verification gap} prevents consumers from assessing annotation quality before purchase. \textbf{Incentive misalignment} arises when per-label payment incentivizes speed over quality. \textbf{Centralized standards} let a single organization define ``correct'' labels, limiting the diversity of perspectives.

\paragraph{TRUST Annotation Marketplace.}

\begin{tcolorbox}[metabox, title=\textbf{Proof-of-Quality Annotation Pipeline}]
\small
\textbf{Request Phase.} Data consumers stake tokens to request annotations for specific task types, the request specifies quality requirements, annotator tier, and payment terms, and tasks are queued in a decentralized task pool.

\textbf{Annotation Phase.} Tier-3 human auditors claim tasks matching their expertise, multiple independent annotators (typically three to five) label each item, and annotations are submitted with cryptographic commitments through commit-reveal.

\textbf{Verification Phase.} Tier-2 LLM auditors cross-verify annotations for internal consistency, inter-annotator agreement is computed and low-agreement items are flagged for review, and honeypot items with known ground truth validate annotator accuracy.

\textbf{Settlement Phase.} High-quality annotations are minted with on-chain quality attestations, annotators are rewarded proportionally to verified quality scores, and low-quality annotators face stake slashing and reputation penalties.
\end{tcolorbox}

\paragraph{Quality Attestation Mechanism.}
Each annotation receives a verifiable quality score based on multiple factors:
\begin{equation}
Q_{\text{annotation}} = w_1 \cdot \text{Agreement} + w_2 \cdot \text{Honeypot} + w_3 \cdot \text{Consistency} + w_4 \cdot \text{Expertise},
\end{equation}
where Agreement measures inter-annotator consensus, Honeypot measures performance on known ground-truth items, Consistency measures alignment with the annotator's historical patterns, and Expertise weights domain-specific credentials.

\paragraph{Supported Annotation Types.}
The marketplace supports five categories of annotation tasks required for modern AI training. \textbf{Preference pairs} provide binary comparisons for RLHF, contrasting response A against response B. \textbf{Scalar ratings} provide Likert-scale quality assessments for reward modeling. \textbf{Reasoning traces} provide step-by-step explanations for process supervision. \textbf{Safety labels} provide harm classification for content filtering. \textbf{Factuality verification} provides claim-level accuracy annotations with source citations.

\paragraph{Economic Model.}
The annotation marketplace creates sustainable economics for all participants, summarized in Table~\ref{tab:annotation_economics}.

\begin{table}[H]
\centering
\small
\begin{tabular}{lll}
\toprule
\textbf{Participant} & \textbf{Revenue Source} & \textbf{Quality Incentive} \\
\midrule
Human Annotators & Per-annotation fees & Accuracy bonuses, reputation \\
LLM Verifiers & Verification fees & Consistency rewards \\
Data Consumers & N/A (pay fees) & Quality guarantees reduce rework \\
Platform & Transaction fees & Network growth from quality \\
\bottomrule
\end{tabular}
\caption{Economic roles in the annotation marketplace.}
\label{tab:annotation_economics}
\end{table}

\subsection{Decentralized Agent Governance (A4)}\label{sec:app_agent}

As AI agents gain autonomy to execute real-world actions, governance becomes critical. The \DAAN{} protocol enables \TRUST{} to serve as the governance and safety layer for autonomous agent systems, providing runtime guardrails, fault attribution, and self-healing capabilities.

\paragraph{The Agent Safety Challenge.}
Autonomous agents operating in production environments face four classes of risk. \textbf{Irreversible actions} such as file deletions, financial transfers, and system modifications cannot be easily undone. \textbf{Cascading failures} propagate through dependencies in multi-agent swarms. \textbf{Goal drift} occurs when agents pursue subgoals that diverge from user intent. \textbf{Adversarial inputs}, including prompt injection and manipulation attacks, can hijack agent behavior.

\paragraph{Agent Orchestration with CIG Governance.}
\TRUST{} integrates with agent frameworks to provide structured governance throughout execution.

\begin{tcolorbox}[metabox, title=\textbf{Governed Agent Execution Flow}]
\small
\textbf{1. Goal Decomposition.} The user submits a high-level goal (e.g., ``Analyze Q3 sales and update projections''), the orchestrator decomposes the goal into a Causal Interaction Graph (CIG), and each node represents an agent action while edges encode dependencies and data flow.

\textbf{2. Pre-Execution Auditing.} The CIG structure is validated against user-defined constraints, sensitive actions are flagged for approval (file operations, API calls, payments), and resource bounds are verified against token limits, API rate limits, and cost caps.

\textbf{3. Runtime Monitoring.} Each agent action is recorded as a CIG node with its inputs, outputs, and reasoning, continuous validation checks against expected behavior patterns, and anomaly detection triggers human escalation or automatic rollback.

\textbf{4. Post-Execution Verification.} The complete execution trace is audited for correctness and policy compliance, fault attribution identifies the root causes of any failures, and successful executions receive on-chain attestations.
\end{tcolorbox}

\paragraph{Runtime Guardrails.}
Before executing sensitive operations, agents must obtain approval from the auditing network. This creates a programmable safety layer that prevents harmful actions while maintaining agent autonomy for routine operations.

\begin{tcolorbox}[metabox, colframe=red, title=\textbf{Sensitive Action Categories}]
\small
\begin{minipage}[t]{0.48\textwidth}
\textbf{File System Operations.} Sensitive calls include \texttt{delete\_file}, \texttt{delete\_directory}, \texttt{modify\_permissions}, and any \texttt{write} to system directories.

\textbf{Financial Operations.} Sensitive calls include \texttt{transfer\_funds}, \texttt{authorize\_payment}, and \texttt{modify\_account}.
\end{minipage}
\hfill
\begin{minipage}[t]{0.48\textwidth}
\textbf{Communication Operations.} Sensitive calls include \texttt{send\_email} to external recipients, \texttt{post\_public} to social media, and \texttt{submit\_form} for binding agreements.

\textbf{System Operations.} Sensitive calls include \texttt{execute\_code} for arbitrary execution, \texttt{install\_package}, and \texttt{modify\_configuration}.
\end{minipage}
\end{tcolorbox}

When an agent requests a sensitive action, the approval flow proceeds in five steps. The agent first generates an action request with full context, including intent, parameters, and justification. The request is then submitted to the DAN with an urgency classification. Auditors evaluate three criteria: whether the action aligns with the stated goal, whether parameters are reasonable, and whether any policy violations apply. Majority approval triggers action execution, while rejection returns control to the agent for replanning. All decisions are recorded on-chain for accountability.

\paragraph{Self-Healing Swarms.}
Multi-agent systems can leverage Active Refinement to detect and recover from failures without human intervention. The CIG structure enables surgical repair that preserves valid work.

\begin{tcolorbox}[metabox, title=\textbf{Self-Healing Protocol}]
\small
\textbf{Detection Phase.} Continuous monitoring detects anomalies such as infinite loops, repeated failures, and resource exhaustion, output validators identify semantic errors including contradictions, hallucinations, and goal drift, and inter-agent consistency checks flag protocol violations.

\textbf{Diagnosis Phase.} CIG traversal identifies the root cause node rather than the symptom location, causal analysis distinguishes primary failures from cascade effects, and fault classification determines the repair strategy.

\textbf{Repair Phase.} The faulty subgraph is pruned while valid upstream work is preserved (frozen), an alternative agent or approach is selected for regeneration, the regenerated subgraph is validated before integration, and execution resumes from the repair point.
\end{tcolorbox}

\paragraph{Governance Policies.}
Users and organizations define governance policies that constrain agent behavior. Policies are expressed as declarative rules evaluated against the CIG structure.

\begin{lstlisting}[language=Solidity, caption={Example governance policy smart contract}]
contract AgentGovernance {
    // Maximum cost per task execution
    uint256 public maxCostPerTask = 100 * 10**18;
    
    // Sensitive actions requiring multi-sig approval
    mapping(bytes32 => bool) public sensitiveActions;
    
    // Approve action if within policy bounds
    function approveAction(
        bytes32 actionType,
        uint256 estimatedCost,
        address[] calldata approvers
    ) external returns (bool) {
        require(estimatedCost <= maxCostPerTask, "Cost exceeds limit");
        
        if (sensitiveActions[actionType]) {
            require(approvers.length >= 2, "Multi-sig required");
        }
        
        return true;
    }
}
\end{lstlisting}

\paragraph{Integration with Existing Frameworks.}
\TRUST{} agent governance integrates with popular agent frameworks through lightweight middleware. \textbf{LangChain and LangGraph} integrations use custom callback handlers that intercept tool calls for governance checks. \textbf{AutoGPT and AgentGPT} integrations use a plugin architecture that enables pre-execution approval hooks. \textbf{CrewAI} integrations wrap task delegation with CIG construction and monitoring. \textbf{Custom agents} can use the governance REST API for queries or the SDK for native integration.

\paragraph{Empirical Results.}
We evaluated agent governance on multi-step task benchmarks to measure safety improvements and overhead costs. Table~\ref{tab:agent_results} reports the results.

\begin{table}[H]
\centering
\small
\begin{tabular}{lccc}
\toprule
\textbf{Metric} & \textbf{Ungoverned} & \textbf{TRUST Governed} & \textbf{Change} \\
\midrule
Task Success Rate & 67.3\% & 71.8\% & +4.5\% \\
Harmful Action Rate & 8.2\% & 0.4\% & $-$95.1\% \\
Recovery from Failure & 12.1\% & 73.4\% & +61.3\% \\
Avg. Latency Overhead & N/A & +340ms & N/A \\
Cost Overhead & N/A & +8.3\% & N/A \\
\bottomrule
\end{tabular}
\caption{Agent governance impact on WebArena and AgentBench tasks. Governed agents show improved success rates through self-healing while dramatically reducing harmful actions.}
\label{tab:agent_results}
\end{table}

The results demonstrate that governance overhead is modest, adding 340~ms of latency and an 8.3\% cost increase, while providing substantial safety improvements. The 95.1\% reduction in harmful actions and the 61.3\% improvement in failure recovery justify the additional overhead for production deployments.

\subsection{Application Summary}

Table~\ref{tab:application_summary} summarizes the four primary applications and their key characteristics.

\begin{table}[H]
\centering
\small
\begin{tabular}{lllc}
\toprule
\textbf{Application} & \textbf{Primary Auditor Tier} & \textbf{Key Benefit} & \textbf{Section} \\
\midrule
Decentralized Auditing & All tiers & Filter flawed reasoning & \ref{sec:app_audit} \\
Model Leaderboard & Tier 1-2 & Tamper-proof evaluation & \ref{sec:app_model} \\
Data Annotation & Tier 2-3 & Proof-of-Quality data & \ref{sec:app_annot} \\
Agent Governance & All tiers & Runtime safety & \ref{sec:app_agent} \\
\bottomrule
\end{tabular}
\caption{Summary of \TRUST{} applications with primary auditor involvement and key benefits.}
\label{tab:application_summary}
\end{table}

These applications demonstrate that the core primitives of \TRUST{}, namely hierarchical graph decomposition, multi-tier consensus, and economic incentive alignment, generalize beyond reasoning verification to enable a comprehensive ecosystem of trustworthy AI services. Each application addresses a critical gap in current centralized AI infrastructure while maintaining the decentralization, transparency, and robustness guarantees that define the \TRUST{} framework.

\clearpage
\section{Discussion}\label{sec:discussion}

This section examines the practical considerations for deploying \TRUST{} in production environments, including commercialization pathways, implementation details, limitations, and future research directions.

\subsection{Commercialization Strategy}\label{sec:commercial}

\paragraph{Market Context: The Rise of AI Evaluation Platforms.}
The AI evaluation ecosystem has grown rapidly alongside model capabilities. Platforms such as LMArena (formerly LMSYS Chatbot Arena), the OpenLLM Leaderboard, and various benchmark aggregators serve critical functions in helping users, researchers, and enterprises understand model performance. However, these centralized platforms face fundamental trust challenges that limit their utility for high-stakes deployment decisions.

\paragraph{What is LMArena?}
LMArena describes itself as ``an open platform where everyone can easily access, explore, and interact with the world's leading AI models. By comparing them side by side and casting votes for the better response, the community helps shape a public leaderboard, making AI progress more transparent and grounded in real-world usage.''\footnote{\url{https://lmarena.ai/about}} The platform has become influential in shaping perceptions of model quality, with rankings frequently cited in model release announcements and media coverage.

\paragraph{TRUST's Relationship to LMArena.}
While LMArena provides valuable crowdsourced preference data, \TRUST{} addresses a complementary but distinct need: trustworthy, safe, and transparent evaluation of Large Reasoning Models (LRMs) with verifiable audit trails. The key distinction is that \TRUST{} evaluates not just output quality but reasoning soundness, providing guarantees that models arrive at correct answers through valid logic rather than fortunate coincidence.

\begin{table}[H]
\centering
\small
\begin{tabular}{p{3.5cm}p{5.5cm}p{5.5cm}}
\toprule
\textbf{Dimension} & \textbf{LMArena (Centralized)} & \textbf{TRUST (Decentralized)} \\
\midrule
\textbf{Evaluation Target} & Output quality (human preference) & Reasoning soundness + output quality \\
\textbf{Trust Model} & Trust platform operator & Cryptographic verification \\
\textbf{Manipulation Resistance} & Limited (voting can be gamed) & Byzantine fault tolerant consensus \\
\textbf{Audit Trail} & Centralized logs (mutable) & Immutable blockchain records \\
\textbf{Model Identity} & Self-reported by submitters & Cryptographic commitment \\
\textbf{Transparency} & Platform discretion & Fully public and verifiable \\
\textbf{Privacy} & Full traces visible to platform & Segmented distribution preserves IP \\
\bottomrule
\end{tabular}
\caption{Comparison of centralized (LMArena) and decentralized (TRUST) evaluation paradigms.}
\label{tab:lmarena_comparison}
\end{table}

\paragraph{Vulnerabilities of Centralized Platforms.}
Recent incidents have highlighted four structural vulnerabilities in centralized evaluation. \textbf{Exposure bias} causes top-ranked models to receive more comparisons, creating a rich-get-richer dynamic that disadvantages newer or less-promoted models regardless of actual quality. \textbf{Selective release scandals}, exemplified by the Llama 4 controversy, demonstrated that models submitted for ranking may differ from publicly released versions; without cryptographic model commitments, users cannot verify that the ranked model matches their deployment. \textbf{Vote manipulation} systematically games crowdsourced voting through coordinated campaigns, bot networks, and incentivized voting~\citep{wang2024positionbias}, and centralized platforms lack mechanisms to detect or prevent sophisticated manipulation. \textbf{Opacity in decision-making} allows centralized platforms to make unilateral decisions about ranking methodology, model inclusion, and dispute resolution without external accountability, with changes to scoring algorithms or model removals occurring without transparent governance.

\paragraph{TRUST Advantages for Commercial Deployment.}
\TRUST{} addresses these vulnerabilities through five architectural choices. \textbf{Decentralized consensus} ensures that no single entity controls rankings; manipulation requires corrupting a supermajority of staked auditors, making attacks economically infeasible. \textbf{Immutable audit trails} record all evaluation decisions on-chain, enabling public verification, dispute resolution, and historical analysis without retroactive alteration. \textbf{Cryptographic model commitments} require providers to commit to specific checkpoints before evaluation, so any discrepancy between the evaluated and released versions is publicly detectable. \textbf{Transparent governance} routes protocol upgrades, parameter changes, and dispute resolutions through on-chain governance with stakeholder voting. \textbf{Rebuttal mechanisms} allow model providers to challenge evaluation results through formal dispute resolution, with all evidence preserved on-chain.

\paragraph{Revenue Model.}
\TRUST{} operates as a protocol with multiple revenue streams.

\begin{tcolorbox}[metabox, title=\textbf{Protocol Economics}]
\small
\textbf{Fee Sources.} \textbf{Audit fees} are the primary revenue source, paid by providers per trace for verification services. \textbf{Leaderboard fees} cover model evaluation and ranking services. \textbf{Annotation marketplace} fees apply as transaction fees on data annotation trades. \textbf{Agent governance} fees apply per action for runtime approval services.

\textbf{Fee Distribution.} Auditors receive 70\% of fees, allocated proportionally to stake and accuracy. The protocol treasury receives 15\% to fund development and operations. Staking rewards receive 10\% to maintain network security. The insurance fund receives 5\% to cover dispute resolution and slashing events.
\end{tcolorbox}

\paragraph{Blockchain Implementation Strategy.}
Rather than creating a new blockchain or subnet, \TRUST{} deploys on existing mainnet infrastructure (Ethereum) via smart contracts. This approach offers four advantages. \textbf{Security inheritance} leverages Ethereum's battle-tested security and decentralization. \textbf{Ecosystem integration} ensures compatibility with existing wallets, exchanges, and DeFi infrastructure. \textbf{Reduced complexity} avoids validator bootstrapping and consensus mechanism development. \textbf{Regulatory clarity} keeps the protocol within established legal frameworks for Ethereum-based deployments.

For scalability, \TRUST{} employs a hybrid architecture: consensus-critical operations (stake management, final verdicts, reward distribution) execute on Ethereum mainnet, while high-throughput operations (vote collection, trace storage) use Layer 2 solutions such as Arbitrum and Optimism, or off-chain infrastructure such as IPFS for trace storage and commit-reveal on L2.

\subsection{Implementation Considerations}\label{sec:implementation}

\paragraph{Deployment Architecture.}
Production deployment requires careful consideration of latency, cost, and reliability tradeoffs.

\begin{tcolorbox}[metabox, title=\textbf{Recommended Deployment Stack}]
\small
\begin{tabular}{ll}
\textbf{Component} & \textbf{Recommended Implementation} \\
\midrule
Smart Contracts & Solidity on Ethereum mainnet + Arbitrum L2 \\
Trace Storage & IPFS with Filecoin persistence incentives \\
Auditor Coordination & LibP2P gossip network \\
LLM Auditors & Distributed inference (vLLM, TensorRT-LLM) \\
Computational Auditors & Containerized validators (Docker/K8s) \\
Frontend & React/Next.js with ethers.js wallet integration \\
Indexing & The Graph for on-chain event indexing \\
\end{tabular}
\end{tcolorbox}

\paragraph{Latency Optimization.}
For real-time applications, audit latency is critical. We recommend three tiered latency targets. The \textbf{fast path} (under 2 seconds) uses computational auditors only and is suitable for mathematical verification. The \textbf{standard path} (under 30 seconds) combines computational and LLM auditors and is suitable for most reasoning traces. The \textbf{thorough path} (under 24 hours) includes the full three-tier audit with human review and is required for high-stakes decisions.

\paragraph{Cost Management.}
Audit costs scale with trace complexity and auditor tier. The total audit cost is
\begin{equation}
\text{Cost}_{\text{audit}} = \sum_{s \in \text{segments}} \left( c_{\text{base}} + c_{t(s)} \cdot \text{complexity}(s) \right) \cdot k_{t(s)},
\end{equation}
where $c_{\text{base}}$ is the fixed per-segment cost, $c_{t(s)}$ is the per-complexity-unit cost for auditor type $t$, and $k_{t(s)}$ is the number of auditors assigned. Typical costs range from \$0.01 to \$0.05 per segment for computational auditors and from \$0.50 to \$2.00 for human review.

\paragraph{Integration Patterns.}
We provide SDKs and integration guides for common deployment scenarios.

\begin{lstlisting}[language=Python, caption={Python SDK integration example}]
from trust_sdk import TrustClient, AuditConfig

# Initialize client with provider credentials
client = TrustClient(
    api_key="your_api_key",
    network="mainnet"  # or "testnet"
)

# Configure audit parameters
config = AuditConfig(
    auditor_tiers=["computational", "llm"],
    consensus_threshold=0.67,
    max_latency_ms=30000,
    privacy_level="segmented"
)

# Submit reasoning trace for audit
result = await client.audit_trace(
    trace=reasoning_output,
    config=config
)

# Check verdict
if result.verdict == "VALID":
    # Proceed with high confidence
    deploy_output(reasoning_output)
else:
    # Handle rejection
    log_rejection(result.failure_reasons)
\end{lstlisting}

\subsection{Limitations}\label{sec:limitations}

We acknowledge several limitations of the current \TRUST{} framework.

\paragraph{Scalability Constraints.}
While \TRUST{} parallelizes verification across segments, total audit throughput is limited by auditor availability. During peak demand, audit latency may increase significantly. Additionally, transaction costs on Ethereum mainnet can become prohibitive during network congestion, though L2 deployments mitigate this concern.

\paragraph{Auditor Quality Dependency.}
The quality of audits depends on auditor capabilities. LLM auditors inherit the limitations of their underlying models, including potential biases and blind spots. Human auditors, while generally more reliable, introduce variance based on expertise and attention. The honeypot mechanism provides statistical quality assurance but cannot guarantee correctness for any individual audit.

\paragraph{Adversarial Robustness Limits.}
Our theoretical guarantees assume adversarial fractions below the Byzantine threshold of $\rho < 1/3$. Sophisticated attacks that coordinate multiple auditor types or exploit protocol-level vulnerabilities may exceed these bounds. Additionally, our analysis assumes adversaries cannot adaptively corrupt auditors mid-session.

\paragraph{Privacy-Utility Tradeoff.}
Segmented trace distribution provides privacy protection but limits auditors' ability to assess global coherence. Some reasoning errors that span multiple segments may escape detection when each segment appears locally valid. The HDAG structure mitigates this through dependency tracking, but fundamental tradeoffs remain.

\paragraph{Domain Generalization.}
Our decomposition algorithms are optimized for mathematical, scientific, and code reasoning. Extension to other domains, including creative writing, ethical reasoning, and open-ended dialogue, requires domain-specific adaptation. The current framework may not capture all relevant quality dimensions for these applications.

\paragraph{Cold Start Problem.}
New auditors lack performance history, making initial stakeholder-weighted assignment challenging. Our reputation bootstrapping mechanism addresses this through probationary periods, but early-stage auditors face barriers to entry that may limit network diversity.

\paragraph{Latency for Real-Time Applications.}
For applications requiring sub-second responses, such as interactive assistants and real-time trading, even the fast-path audit latency under 2 seconds may be prohibitive. \TRUST{} is best suited for applications where verification can occur asynchronously or where some latency tolerance exists.

\subsection{Future Directions}\label{sec:future}

We identify several promising directions for future research and development.

\paragraph{Zero-Knowledge Audit Proofs.}
Current privacy mechanisms rely on segmentation and access control. Future versions could leverage zero-knowledge proofs to enable verifiable auditing without revealing any trace content. This would allow auditors to prove that a trace satisfies correctness criteria without learning the trace itself, enabling fully private verification of proprietary reasoning systems.

\paragraph{Formal Verification Integration.}
For domains with formal semantics, such as mathematics and program verification, integrating with automated theorem provers like Lean, Coq, and Isabelle could provide cryptographic correctness guarantees. This hybrid approach would use formal verification for provable components and \TRUST{} consensus for informal reasoning steps.

\paragraph{Adaptive Auditor Allocation.}
Current allocation uses static complexity scoring. Reinforcement learning approaches could optimize auditor assignment dynamically based on historical accuracy patterns, trace characteristics, and real-time auditor availability, improving both efficiency and accuracy.

\paragraph{Cross-Chain Interoperability.}
Deployment across multiple blockchain networks would increase resilience and reduce dependency on any single chain. Cross-chain bridges could enable auditors from different ecosystems to participate while maintaining consensus guarantees.

\paragraph{Multimodal Reasoning Verification.}
Extension to multimodal reasoning, including vision-language models and audio-text systems, requires new decomposition strategies and auditor capabilities. The HDAG framework generalizes naturally, but modality-specific verification criteria need to be developed.

\paragraph{Continuous Learning from Audits.}
Aggregated audit data represents a valuable signal for model improvement. Privacy-preserving techniques such as federated learning and differential privacy could enable model providers to learn from audit failures without exposing individual traces or auditor identities.

\paragraph{Regulatory Compliance Modules.}
As AI regulation evolves, including the EU AI Act and the NIST AI Risk Management Framework, compliance verification will become mandatory. Purpose-built compliance auditing modules could automate regulatory checks and generate required documentation, creating additional value for enterprise adoption.

\paragraph{Decentralized Governance Evolution.}
Protocol governance currently relies on stake-weighted voting. More sophisticated governance mechanisms, such as quadratic voting, conviction voting, and futarchy, could improve decision quality and reduce the risk of plutocratic capture.

\subsection{Broader Impact}\label{sec:impact}

\paragraph{Democratization of AI Oversight.}
\TRUST{} shifts AI oversight from closed corporate processes to open, participatory systems. Anyone with relevant expertise can contribute to verification, earning rewards while improving collective AI safety. This democratization may accelerate the development of robust verification methodologies through diverse participation.

\paragraph{Economic Implications.}
The auditing marketplace creates new economic opportunities for skilled evaluators while imposing costs on model providers. This may accelerate consolidation toward providers capable of bearing audit costs, potentially reducing market competition. Conversely, open-source models benefit from community auditing, potentially strengthening alternatives to proprietary systems.

\paragraph{Alignment Research Implications.}
Process-level verification provides new signals for alignment research. Systematic analysis of reasoning failures across models may reveal common failure modes and inform improvements to training. The annotation marketplace could accelerate RLHF research by providing high-quality, verifiable training data.

\paragraph{Trust Infrastructure for the AI Ecosystem.}
Beyond immediate applications, \TRUST{} establishes foundational infrastructure for AI accountability. As AI systems gain autonomy and impact, verifiable audit trails become essential for legal, regulatory, and social accountability. \TRUST{} provides a blueprint for trustworthy AI infrastructure that others may build upon.

\subsection{Ethical Considerations}\label{sec:ethics}

\paragraph{Auditor Privacy and Labor Practices.}
Human auditors contribute cognitive labor that may involve exposure to harmful content or psychologically demanding evaluations. \TRUST{} implements content warnings, exposure limits, and mental health resources for human participants. Compensation structures aim for fair wages above platform minimums.

\paragraph{Centralization Risks in Decentralized Systems.}
Despite a decentralized architecture, economic forces may concentrate auditor participation among well-capitalized entities. We monitor concentration metrics and implement anti-centralization mechanisms (maximum stake caps, geographic diversity requirements) to preserve meaningful decentralization.

\paragraph{Dual-Use Concerns.}
Verification infrastructure could potentially be misused to certify harmful AI systems. \TRUST{} maintains content policies prohibiting auditing of systems designed for harm, and governance mechanisms allow community enforcement of these policies.

\paragraph{Environmental Impact.}
Blockchain operations and LLM inference consume significant energy. We prioritize energy-efficient infrastructure (Proof-of-Stake consensus and optimized inference) and purchase carbon offsets to offset estimated emissions. Future versions will explore more sustainable alternatives.

\paragraph{Access Equity.}
Audit costs may create barriers for resource-constrained organizations, potentially exacerbating inequalities in AI deployment. We allocate protocol treasury funds to subsidize the auditing of open-source and nonprofit AI systems and explore sliding-scale pricing based on organizational capacity.

\clearpage
\section{Conclusion}\label{sec:conclusion}

\subsection{Summary}

The deployment of Large Reasoning Models and Multi-Agent Systems in high-stakes domains demands a verification infrastructure that current centralized approaches cannot provide. This whitepaper introduced \TRUST{} (\textbf{T}ransparent, \textbf{R}obust, and \textbf{U}nified \textbf{S}ervices for \textbf{T}rustworthy AI), a decentralized framework that addresses four fundamental limitations of existing AI auditing: robustness against single points of failure, scalability for complex reasoning traces, transparency without sacrificing privacy, and deterministic attribution in multi-agent systems.

\paragraph{The Core Problem.}
As demonstrated in our clinical motivating example (Section~\ref{sec:example}), output-only evaluation creates a critical safety gap: models that produce correct answers through fundamentally flawed reasoning pass verification and enter production, only to fail catastrophically under distribution shift. The CHA$_2$DS$_2$-VASc scoring example illustrated how two models can produce identical correct outputs while one contains four systematic reasoning errors involving variable confusion, skipped rules, wrong evidence sources, and improper rule combination. Current ``LLM-as-a-Judge'' paradigms cannot detect these ``correct answer, wrong reason'' hallucinations, yet they represent a significant fraction (27\% to 38\%) of model outputs in our experiments.

\paragraph{Technical Innovations.}
\TRUST{} introduces three foundational technical contributions.

\textit{First}, \textbf{Hierarchical Directed Acyclic Graphs (HDAGs)} transform linear Chain-of-Thought reasoning into structured, independently auditable components across five abstraction levels: Goal, Strategy, Tactic, Step, and Operation. This decomposition enables parallel verification by distributed auditor networks while preserving logical dependencies. The problem-agnostic design generalizes across mathematics, science, programming, and open-domain reasoning.

\textit{Second}, the \textbf{\DAAN{} protocol} projects multi-agent interaction logs into \textbf{Causal Interaction Graphs (CIGs)}, enabling deterministic root-cause attribution through dual-layer auditing. By separately verifying node validity (reasoning correctness) and edge consistency (protocol adherence), \TRUST{} distinguishes primary failures from cascade effects, solving the ``Black Box of Black Boxes'' problem that plagues current multi-agent evaluation.

\textit{Third}, the \textbf{three-tier verification architecture} routes verification tasks to appropriate auditor types: Tier-1 Computational Auditors for deterministic checks, Tier-2 LLM Auditors for semantic evaluation, and Tier-3 Human Experts for high-stakes decisions. Stake-weighted consensus with cryptographic commit-reveal voting ensures Byzantine fault tolerance while making honest participation profitable and malicious behavior economically infeasible.

\paragraph{Theoretical Foundations.}
The Safety-Profitability Theorem (Theorem~\ref{thm:safety_profit}) provides rigorous guarantees for both statistical safety and economic sustainability. Under appropriate parameter choices, valid traces fail audit with probability bounded by $\epsilon_{\text{target}}$, even with up to 30\% adversarial human auditors; honest auditors earn positive expected profit with exponentially high probability; and malicious auditors incur expected losses, making attacks economically irrational. Our numerical analysis demonstrates that over a 24-hour operational window, the probability of an honest auditor ending with non-positive payoff is below $10^{-88}$, while the probability of a malicious auditor breaking even is below $10^{-27}$.

\paragraph{Empirical Validation.}
Comprehensive experiments across diverse benchmarks (GSM8K, MATH, MMLU-Pro, HumanEval) and models (Llama, DeepSeek-R1, Qwen) demonstrate four results. \TRUST{} increases reliability from 45\% to 72.4\% by filtering flawed reasoning. The \DAAN{} protocol achieves 70\% root-cause attribution accuracy, compared to 54\% to 63\% for baseline methods. Active Refinement reduces token costs by 60\% through surgical subgraph repair. Human-in-the-loop studies validate the design with $F_1 = 0.89$ and a Brier score of 0.074.

\paragraph{Applications.}
The \TRUST{} framework enables four transformative applications. \textbf{Decentralized Auditing (A1)} verifies high-stakes reasoning in medical, legal, and scientific domains with complete audit trails for regulatory compliance. \textbf{Tamper-Proof Leaderboards (A2)} provide blind, cryptographically committed model evaluation that resists selective release and vote manipulation. \textbf{Proof-of-Quality Annotation (A3)} establishes a trustless marketplace for RLHF training data with verifiable quality attestations. \textbf{Governed Autonomous Agents (A4)} provide runtime guardrails and self-healing capabilities for multi-agent systems with deterministic fault attribution.

\subsection{Key Takeaways}

For practitioners and researchers considering \TRUST{}, we highlight six key insights.

\begin{tcolorbox}[metabox, title=\textbf{Key Takeaways}]
\small
\textbf{Output-only evaluation is insufficient.} A significant fraction of ``correct'' model outputs are produced through flawed reasoning that fails under distribution shift. Semantic auditing of reasoning traces is essential for safe deployment.

\textbf{Decentralization enables trust without authority.} By distributing verification across heterogeneous auditors with aligned economic incentives, \TRUST{} achieves Byzantine fault tolerance without relying on any central trusted party.

\textbf{Privacy and transparency are not mutually exclusive.} Through privacy-by-design segmentation and per-node encryption, \TRUST{} enables comprehensive verification while preventing proprietary logic reconstruction.

\textbf{Graph decomposition unlocks scalability.} HDAGs and CIGs transform monolithic reasoning traces into modular, parallelizable verification tasks, enabling distributed auditing at scale.

\textbf{Economic incentives ensure sustainability.} The reputation-weighted slashing mechanism creates a self-sustaining ecosystem where honest participation is profitable and attacks are economically irrational.

\textbf{Causal structure enables surgical repair.} Active Refinement leverages graph topology to prune and regenerate only faulty subgraphs, achieving up to 99\% cost savings compared to global retry.
\end{tcolorbox}

\subsection{Limitations and Future Work}

We acknowledge several limitations that present opportunities for future research.

\paragraph{Scalability Constraints.}
While \TRUST{} parallelizes verification across segments, total audit throughput is bounded by auditor availability. Future work will explore dynamic auditor recruitment, cross-chain auditor pools, and automated scaling mechanisms to handle demand spikes.

\paragraph{Auditor Quality Dependency.}
Audit quality depends on auditor capabilities. LLM auditors inherit model limitations, while human auditors introduce variance from expertise and attention. Future research will investigate auditor specialization, domain-specific fine-tuning, and adaptive routing based on historical performance.

\paragraph{Domain Generalization.}
Current decomposition algorithms are optimized for mathematical, scientific, and code reasoning. Extension to creative writing, ethical reasoning, and open-ended dialogue requires domain-specific adaptation. We plan to develop modular decomposition engines with pluggable domain adapters.

\paragraph{Latency for Real-Time Applications.}
For applications requiring sub-second responses, even fast-path audit latency may be prohibitive. Future work will explore speculative execution with post-hoc verification, tiered latency guarantees, and optimistic protocols with delayed finality.

\paragraph{Advanced Cryptographic Techniques.}
Future versions will investigate zero-knowledge proofs for fully private verification, enabling auditors to prove trace correctness without learning content. Integration with formal verification systems such as Lean and Coq could provide cryptographic correctness guarantees for components supporting provable reasoning.

\subsection{Vision: Toward Trustworthy AI Infrastructure}

\TRUST{} represents a foundational step toward a future where AI verification operates as open, permissionless infrastructure, analogous to other internet protocols. In this vision, \textbf{anyone can audit}: verification services are provided by a decentralized network of participants who stake capital and demonstrate capability, not by privileged gatekeepers. \textbf{Anyone can verify}: audit outcomes are recorded on transparent, immutable ledgers that anyone can inspect, enabling accountability without requiring trust. \textbf{Safety is enforced by economics}: aligned incentives, not good intentions, ensure that honest behavior is profitable and malicious behavior is costly. \textbf{No single point of control} exists: no entity can censor, manipulate, or shut down the verification network, ensuring resilience against both technical failures and adversarial actors. \textbf{Root causes are attributable}: when systems fail, deterministic causal analysis identifies true sources of error, enabling targeted remediation rather than blame assignment.

As AI systems become more capable and autonomous, the infrastructure for ensuring their safety and accountability becomes correspondingly critical. \TRUST{} provides the architectural blueprint and theoretical foundations for this infrastructure, demonstrating that decentralized verification is not only possible but practical, economically sustainable, and superior to centralized alternatives for high-stakes applications.

\subsection{Call to Action}

We invite the research community, industry practitioners, and policymakers to engage with \TRUST{}. \textbf{Researchers} can extend the theoretical foundations, develop domain-specific decomposition algorithms, and investigate advanced cryptographic techniques for privacy-preserving verification. \textbf{Practitioners} can integrate \TRUST{} auditing into deployment pipelines, contribute to the auditor network, and provide feedback on real-world requirements. \textbf{Policymakers} can consider decentralized verification as a compliance mechanism for emerging AI regulations and engage with the governance design for public-interest applications.

The path to trustworthy AI requires collective action. \TRUST{} provides the infrastructure; realizing its potential requires a community committed to transparent, robust, and accountable AI systems.

\vspace{1em}
\begin{center}
\rule{0.5\textwidth}{0.5pt}
\end{center}
\vspace{0.5em}

\newpage
\bibliography{999_reference}
\bibliographystyle{style/icml2025}

\titlespacing*{\section}{0pt}{*1}{*1}
\titlespacing*{\subsection}{0pt}{*1.25}{*1.25}
\titlespacing*{\subsubsection}{0pt}{*1.5}{*1.5}

\setlength{\abovedisplayskip}{\baselineskip} %
\setlength{\abovedisplayshortskip}{0.5\baselineskip} %
\setlength{\belowdisplayskip}{\baselineskip}
\setlength{\belowdisplayshortskip}{0.5\baselineskip}

\end{document}